\documentclass{article}

\usepackage{microtype}
\usepackage{graphicx}
\usepackage{subcaption}
\usepackage{wrapfig}
\usepackage{booktabs} 
\usepackage{stmaryrd}
\usepackage{hyperref}

\usepackage[preprint]{icml2026}

\usepackage{amsmath}
\usepackage{amssymb}
\usepackage{mathtools}
\usepackage{amsthm}

\usepackage[capitalize,noabbrev]{cleveref}

\theoremstyle{plain}
\newtheorem{theorem}{Theorem}[section]
\newtheorem{proposition}[theorem]{Proposition}
\newtheorem{lemma}[theorem]{Lemma}

\theoremstyle{definition}
\newtheorem{definition}[theorem]{Definition}

\theoremstyle{remark}
\newtheorem{remark}[theorem]{Remark}
\newcommand{\R}{\mathbb{R}}
\newcommand{\Id}{{\operatorname{Id}}}
\usepackage[textsize=tiny]{todonotes}

\icmltitlerunning{Sliced ReLU Attention}
\begin{document}

\twocolumn[
  \icmltitle{Sliced ReLU attention: \\ Quasi-linear contextual expressivity via  sorting}



  \icmlsetsymbol{equal}{*}

  \begin{icmlauthorlist}
    \icmlauthor{Fran\c{c}ois-Xavier Vialard}{yyy}
    \icmlauthor{Siwan Boufadène}{yyy}
  \end{icmlauthorlist}

  \icmlaffiliation{yyy}{LIGM, Univ. Gustave Eiffel.
  }

  \icmlcorrespondingauthor{Fran\c{c}ois-Xavier Vialard}{francois-xavier.vialard@univ-eiffel.fr}
  \icmlcorrespondingauthor{Siwan Boufad\`ene}{siwan.boufadene@univ-eiffel.fr}

  \vskip 0.3in
]



\printAffiliationsAndNotice{}  

\begin{abstract}
        We introduce sliced ReLU attention, a new attention mechanism that departs structurally from both softmax and its approximation alternatives. Instead of applying a nonlinearity to pairwise dot products, we operate on one-dimensional projections of key–query \emph{differences} and leverage sorting to obtain quasi-linear complexity. This construction yields a differentiable, non-symmetric kernel that can be computed in $O(n\log(n))$ through a sorting procedure, making it suitable for very long contexts. 
        Beyond computational benefits, the model retains strong theoretical expressive power: we establish two in-context expressivity results, previously known for softmax attention, showing that sliced ReLU attention preserves the ability to perform nontrivial sequence-to-sequence disentangling tasks and satisfies a contextual universal approximation property. 
        Finally, we illustrate the potential practical interest of this kernel in small to medium-scale experiments.
\end{abstract}

\section{Introduction}
\label{sec:introduction}

Since their introduction, Transformers \cite{NIPS2017_3f5ee243} underpin most state-of-the-art neural network architectures in natural language processing and computer vision \cite{NEURIPS2020_1457c0d6, devlin2019bertpretrainingdeepbidirectional,dosovitskiy2021imageworth16x16words,liu2021swintransformerhierarchicalvision,baevski2020wav2vec20frameworkselfsupervised,chen2022beatsaudiopretrainingacoustic}. 

A key limitation of standard attention is its quadratic computational cost with respect to the input length. 
There have been several approaches to tackle this computational bottleneck. One direction focuses on more efficient implementations of standard softmax attention, such as FlashAttention \cite{dao2023flashattention2fasterattentionbetter}. Although these methods optimize memory usage and constant factors, their overall complexity remains quadratic in the input length. A second line of methods seeks to approximate the attention matrix, either through random features \cite{choromanski2022rethinkingattentionperformers}, low-rank factorization \cite{winata2020lightweightefficientendtoendspeech, hu2022lora}, or sparse or local mechanisms such as sliding windows, clustered attention or global-local mixtures \cite{beltagy2020longformerlongdocumenttransformer,vyas2020fasttransformersclusteredattention,zaheer2021bigbirdtransformerslonger}. The linearization of the softmax around the origin leads to linear attention architectures \cite{katharopoulos2020transformersrnnsfastautoregressive,qin2022cosformerrethinkingsoftmaxattention}, where the attention output is computed using kernelized feature maps via the dot product of two learnable maps, scaling linearly with sequence length. While fast, these models often suffer from degraded performance on tasks requiring rich global interactions. Other attention-free architectures have been proposed, such as state space sequence models \cite{gu2022efficientlymodelinglongsequences, goel2022itsrawaudiogeneration} or convolutional models \cite{liu2022convnet2020s}, which sometimes match or outperform Transformers on long-sequence tasks. We refer to \cite{tay2022efficienttransformerssurvey,arora2025simplelinearattentionlanguage} for a more in-depth review of efficient models.

From a theoretical perspective, recent analyses have provided deeper insight into the success of softmax-based attention architectures across diverse tasks. Transformers have been shown to possess strong universal approximation properties, enabling them to represent arbitrary sequence-to-sequence mappings \cite{yun2020transformersuniversalapproximatorssequencetosequence}. More recent theoretical works study the \emph{mean-field} limit of Transformers, extending attention to infinitely long contexts represented as probability measures in a Euclidean space \cite{vuckovic2020mathematicaltheoryattention, sander2022sinkformerstransformersdoublystochastic}. The optimal transport framework has been used to study the smoothness of Transformers in measure spaces \cite{castin2024smoothattention}, and universality results have been generalized to such spaces \cite{furuya2024transformersuniversalincontextlearners}. Similar properties hold for other attention computations, such as sigmoid-based variants in \cite{ramapuram2025theoryanalysisbestpractices}. Transformers have also been studied in a continuous-depth limit, where they act as universal measure-to-measure interpolators \cite{geshkovski2024measuretomeasureinterpolationusingtransformers}. Clustering properties of the resulting attention dynamics are studied in \cite{geshkovski2024emergenceclustersselfattentiondynamics, vuckovic2020mathematicaltheoryattention}.

In our main contribution, we utilize the slicing technique, also known as one-dimensional projections, which have been well-explored for comparison of measures \cite{NEURIPS2019_f0935e4c,hertrich2024generativeslicedmmdflows} or kernel computations \cite{2025arXiv251011478R,hertrich2024fast}. This idea has also been put forward to post-process the standard attention matrix to transform it into a bi-stochastic matrix in \cite{Shahbazi2025ESPFormerDA} to reduce the computational burden of \cite{sander2022sinkformerstransformersdoublystochastic}. 
Rather than viewing slicing primarily as a computational approximation, we use it to investigate a new family of kernel-based attention models.
Finally, the technique of differential sorting, together with the Sinkhorn algorithm, has been used in \cite{pmlr-v119-tay20a} to propose an efficient sparse approximation of the softmax attention.
Another close model of efficient attention is the reformer architecture \cite{kitaev2020reformerefficienttransformer} that uses an approximation of the standard attention based on locality sensitive hashing, which results in an $n \log(n)$ complexity. 
 In \cite{yuan2023sliceformermakemultiheadattention}, slicing and sorting are used to produce an attention mechanism that encodes implicit attention through a permutation matrix, which suffers from limitations in expressivity. In sharp contrast with these directions, we explore a different model of attention, based on sorting and slicing, with expressivity properties.

\noindent{\textbf{Main contributions:} }
Our work aligns with the broader effort to investigate alternative attention mechanisms from a kernel-based perspective, departing from the softmax function. 
Somewhat surprisingly, given the importance of attention in modern architectures, relatively few fundamentally distinct alternatives have been proposed.
Moreover, the reasons for the empirical dominance of softmax attention against its competitors remain poorly understood, although a few hypotheses have been recently proposed \cite{arora2023zoologymeasuringimprovingrecall,duranthon2025statisticaladvantagesoftmaxattention,shen2024scalinglawslinearcomplexity}.
Most of these alternatives are driven by the goal of mitigating the quadratic computational cost inherent to attention. However, in one-dimensional settings, this footprint can sometimes be reduced to quasi-linear time through sorting. Motivated by the possibility of generating an attention-like mechanism through sorting,
we introduce a new quasi-linear Transformer architecture, based on a sliced kernel computation, extending ideas from \cite{boufadene2025fastlargedeformationmatching}. In contrast to other efficient architectures, our method enables exact, global, and expressive kernel-based interactions while maintaining a quasi-linear computational cost and linear memory usage.

From a theoretical standpoint, we show that our model satisfies the same universality properties as those established for softmax attention in \cite{furuya2024transformersuniversalincontextlearners}, with depth as the only asymptotically growing parameter. We also prove that our proposed ReLU attention is a universal approximator of sequence-to-sequence functions, in the same way that softmax attention is \cite{yun2020transformersuniversalapproximatorssequencetosequence}.

Finally, we provide an empirical validation of the method on a range of standard Transformer benchmarks. We first evaluate sliced ReLU attention on the Long Range Arena benchmark \cite{tay2020longrangearenabenchmark}, which is specifically designed to test the ability to model long-context reasoning. We also provide a small-scale experiment on Vision Transformer.  We then examine the behavior of our model on geometric classification problems such as ModelNet40 \cite{Wu2015}, using a Point Cloud Transformer architecture \cite{Guo2021PCT}. We also demonstrate that sliced ReLU attention supports medium-scale masked language model pretraining by training a BERT-like encoder \cite{devlin2019bertpretrainingdeepbidirectional}, followed by downstream evaluation on the GLUE benchmark \cite{gluedataset_2018}, as a validation of practical compatibility with modern Transformer training pipelines.


\noindent{\textbf{Perspectives:} }
On the empirical side, it would be relevant to evaluate the model on larger-scale pretraining, multimodal contexts, or generative settings, where global and efficient computations are critical. A finer study of the difference between standard softmax attention and our ReLU-based version would also be valuable, as the two architectures may exhibit different strengths depending on the task and structure of the input.
From a theoretical standpoint, the fact that the expressivity results are on par with softmax attention suggests that further progress may hinge on understanding and optimizing the training dynamics of this attention mechanism. In particular, the interplay between embedding dimension and head count appears central for capturing the full capacity of sliced ReLU attention, and clarifying these trade-offs is an important direction for future work. 

\section{Model}\label{sec:model}
In this section, we introduce our ReLU attention models\footnote{The source code is 
available 
at \url{https://github.com/SiwanB/Sliced_ReLU_Attention.git}.}
as a scalable alternative to the standard quadratic in time softmax attention. 

\subsection{Standard Attention}
In Transformers, the standard attention mechanism computes a relevance-weighted combination of value vectors using the softmax function. More precisely, given parameters $\theta = (Q,K,V) \in \R^{3 \times d\times d}$ (possibly affine transformations) and an input sequence $X = (x_1, \ldots, x_n) \in \R^{n\times d}$, an attention head computes
\vspace{-0.3em}
\begin{equation}\label{eq:softmax_att}
\mathcal{A}_\theta^{\text{sftm}}(x_i,(x_1, \ldots, x_n)) \coloneqq \sum_{j=1}^n \frac{\exp(\langle Q x_i,K x_j\rangle)}{\sum_k \exp(\langle Qx_i,Kx_k\rangle)}V x_j \,.
\end{equation}
While expressive, computing all pairwise interactions induces a quadratic cost in the sequence length $n$, which quickly becomes a bottleneck for long contexts. 

A useful observation is that when a Layer Norm is applied to the inputs \cite{xiong2020layernormalizationtransformerarchitecture}, the quantity $\exp(\langle Q x_i,K x_j\rangle)$ behaves similarly to the Gaussian kernel $\exp(-\| Q x_i - K x_j \|^2)$ once the vectors lie approximately on a sphere. This kernel viewpoint motivates the search for alternative non-negative and geometrically meaningful kernels that may be cheaper to compute. 

\subsection{Sliced ReLU Attention}
Based on the previous observations, we propose an alternative attention mechanism based on a \textbf{projected ReLU kernel}, which preserves non-negativity and expressivity while allowing efficient computation. We use the same parameters $\theta = (Q,K,V) \in \R^{3 \times d\times d}$ (or affine transformation), and use a projection operator $\Pi : \R^d \mapsto \R$, where $\Pi$ may be a linear projection or a small MLP. Given these parameters, our ReLU attention head computes
\vspace{-0.3em}
\begin{equation}\label{eq:relu_att}   \mathcal{A}_{\theta,\Pi}^{\text{ReLU}}(x_i,(x_1, \ldots, x_n)) \coloneqq \sum_{j=1}^n \frac{\operatorname{ReLU}(\Pi Q x_i - \Pi K x_j)}{\sum_{l=1}^n |\Pi Q x_i - \Pi K x_l|}V x_j \,.
\end{equation}
Here, the $\operatorname{ReLU}$ function acts as a non-negative, difference-based interaction kernel. In contrast to random-feature or dot-product–based approximations, our mechanism relies on projecting queries and keys to one-dimensional scores and applying ReLU to their differences, which is the key ingredient enabling an efficient sliced computation. Several normalizations are possible. Empirically, we found that using the absolute difference $|\Pi Q x_i - \Pi K x_l|$ yields more stable training than normalizing by $\operatorname{ReLU}(\Pi Q x_i - \Pi K x_l)$, whose value may vanish or fluctuate sharply when projected scores are close. This scale-based normalization consistently outperformed enforcing a probability simplex normalization as in softmax. Determining an optimal normalization remains an open question.

The geometrical interpretation is the following: keys with low projected scores (i.e. large negative values relative to a query) attend to most queries, while keys with high projected scores influence only a few. In our experiments, $\Pi$ is implemented as a one-hidden-layer MLP mapping tokens in $\mathbb R^d$ to $\mathbb R$ (or to $\R^H$ for multi-head attention), with hidden dimension equal to $d$ to keep the computational overhead and parameter count minimal. This projection defines the slicing direction along which the ReLU interaction is computed, and different learned MLPs produce different interaction profiles.


\subsubsection{Centered value parametrization for ReLU attention}

During training, we consistently observe better performance when the value vectors are centered, i.e. using $Vx_j - \frac{1}{n}\sum_l V x_l$. This operation enforces a zero-sum constraint on the values. This phenomenon can be explained by a structural property of the ReLU kernel. Although asymmetric, the kernel $(x,y) \mapsto \operatorname{ReLU}(x-y)$ is conditionally positive definite of order 1: it induces a positive definite quadratic form on vectors $\gamma \in \R^n$ that verify $\sum_i \gamma_i = 0$ \cite{Wendland_2004,auffray_pdcpker_2009}. A short derivation (see Appendix \ref{app:cpd_relu}) shows that this follows from the conditional positive definiteness of the Energy Distance kernel.

This insight explains the empirical benefit of using centered values: it ensures that the attention output is computed within the subspace where the ReLU kernel behaves in a stable, positive-definite manner. The resulting effective attention computation becomes:
\begin{multline}\label{eq:relu_att_centered}   \mathcal{A}_{\theta,\Pi}^{\text{ReLU}}(x_i,(x_1, \ldots, x_n)) \coloneqq \\ \sum_{j=1}^n \frac{\operatorname{ReLU}(\Pi Q x_i - \Pi K x_j)}{\sum_{l=1}^n |\Pi Q x_i - \Pi K x_l|}(Vx_j - \frac{1}{n}\sum_l V x_l) \,.
\end{multline}

\subsubsection{Quasi-Linear complexity}\label{SecQuasiLinearComputation}
Although a naive implementation of our ReLU attention requires $O(n^2)$ operations, the structure of the ReLU kernel allows the full attention computation to be performed in $O(n \log n)$ time.
The key observation is that the ReLU kernel can be written as an asymmetric variant of the Energy Distance kernel in dimension one. For any $x,y \in \R$:
\begin{equation}\label{eq:relu_is_ed}
    \operatorname{ReLU}(x-y) = \frac{|x-y|}{2} + \frac{x-y}{2} \,.
\end{equation}
This identity makes it possible to reuse the sliced computation strategy of \cite{boufadene2025fastlargedeformationmatching}, while adapting it to the ReLU kernel.

\begin{proposition}
    Given inputs $(x_1, \ldots, x_n) \in \R^{n \times d}$, parameters $\theta$, and a projection operator $\Pi$, the vector of ReLU attention 
    \begin{equation}
        (\mathcal{A}_{\theta,\Pi}^{\operatorname{ReLU}}(x_i,(x_1, \ldots, x_n)), 1 \leq i \leq n) \in \R^{n \times d}
    \end{equation}
    can be computed in $O(n \log n )$ time.
\end{proposition}
\begin{proof}
    Let $z_1 \leq z_2 \leq \ldots \leq z_N$ be sorted projection scores, with associated values $\gamma_1, \ldots, \gamma_N \in \R^d$. The ReLU interaction admits the simple identity, for $1 \leq i \leq N$:
    \vspace{-0.3em}
    \begin{multline}
        \sum_{j=1}^N \operatorname{ReLU}(z_i - z_j) \gamma_j = \sum_{j\leq i} (z_i - z_j)\gamma_j \\ = z_i\sum_{j \leq i}\gamma_j - \sum_{j \leq i} z_j \gamma_j = a_i z_i - b_i\,.
    \end{multline}
    The cumulative sums $a_i$ and $b_i$ can be computed in linear time, via the formula: $a_0 = 0, a_{i} = a_{i-1} + \gamma_i$
and  $b_0 = 0,b_{i} = b_{i-1} + \gamma_i z_i$.   
    Computing the $\operatorname{ReLU}$ convolution of unsorted data amounts to sorting it, computing the corresponding $v_i$, then unsorting it back to its original order. Sorting is an $O(n \log n )$ operation, which dominates the runtime.
    In practice, the attention mechanism uses both keys and queries. To compute all interactions in a single sorted pass, we concatenate the projected keys and queries into a single array $(k_1, \ldots, k_n, q_1, \ldots, q_n)$, assign $\gamma_i = v_i$ for $1 \leq i \leq n$, $\gamma_i = 0$ for $n+1 \leq i \leq 2n$, and apply the above linear-time computation to this unified sequence. After the convolution is computed on the sorted array, the result is unsorted back to the original order, and the outputs corresponding to the query indices $n+1 \leq i \leq 2n$ yield the attention value. The denominator being the symmetrization of the ReLU kernel, the computational complexity is the same.
\end{proof}


\begin{remark}
    While our attention is global and exact, rather than sparse, local or approximate, it still runs in $O(n\log n)$ time, in contrast with the quadratic complexity of standard softmax attention. This efficiency comes from leveraging sorting-based interactions, a well-studied and hardware-optimized primitive. Moreover, unlike previous sorting-based attention models such as SliceFormer \cite{yuan2023sliceformermakemultiheadattention}, our computations remain fully differentiable, making it compatible with smooth gradient-based optimization.
\end{remark}

\subsection{Sliced ReLU-Bump Attention}

In some applications, especially those requiring sharper and more localized token interactions, the one-sided ReLU kernel may be too diffuse. To address this, we introduce a \textbf{sliced ReLU-bump kernel}, obtained by combining shifted copies of the original sliced ReLU interaction to produce a compact, symmetric bump. 
This operation is a simple modification of sliced ReLU attention: it is a linear combination of three ReLU interactions (see Appendix, Lemma \ref{lemma:relu_bump}).
For a bandwidth parameter $b>0$, we define the sliced ReLU-bump attention (or hat kernel):
\begin{multline}\label{eq:relu_bump_att}
    \mathcal{A}_{\theta,\Pi}^{\text{ReLU-Bump}}(x_i,(x_1, \ldots, x_n)) \coloneqq \\ \frac{1}{n} \sum_{j=1}^n \operatorname{ReLU}\left(1-\frac{|\Pi Q x_i - \Pi K x_j|}{b}\right)Vx_j \,.
\end{multline}
A key implementation detail is that, in contrast to the nonlinear projection used in our previous sliced ReLU attention, the ReLU-bump variant employs a \textbf{linear projection} $\Pi$. We also use the simple $1/n$ normalization shown above, rather than the more structured normalization used in the sliced ReLU formulation, and we do not need to center the value vectors to retain performance. 
Indeed, this kernel is positive definite\footnote{Its Fourier transform is nonnegative since it is the square of a \emph{sinc} function.} in one dimension (see Appendix \ref{AppendixWendland}).
This keeps the mechanism lightweight while still enabling localized interactions.

    In 1D, this construction yields a compact, triangular interaction profile, the so-called hat function, analogous to a localized bump function, that preserves non-negativity. Crucially, each of the three ReLU terms admits the same cumulative-sum expression as the original sliced ReLU attention. Therefore, the kernel can be computed in a single sort-and-scan pass, retaining an $O(n\log n)$ time complexity. In higher dimensions, the resulting kernel induces a narrow interaction band orthogonal to the projection direction.
Empirically, this sharpened interaction pattern proves advantageous in tasks requiring \emph{fine-grained, local geometric structure}, such as point-cloud classification on ModelNet40 or dense token-level objectives in masked language modeling and NLP tasks. In these settings, the ReLU-bump kernel offers more precise control over local influence while preserving global, differentiable, and efficient sliced mechanism.
\begin{figure}[htbp]
    \centering

    \begin{subfigure}[t]{0.49\linewidth}
        \centering
        \includegraphics[width=\linewidth]{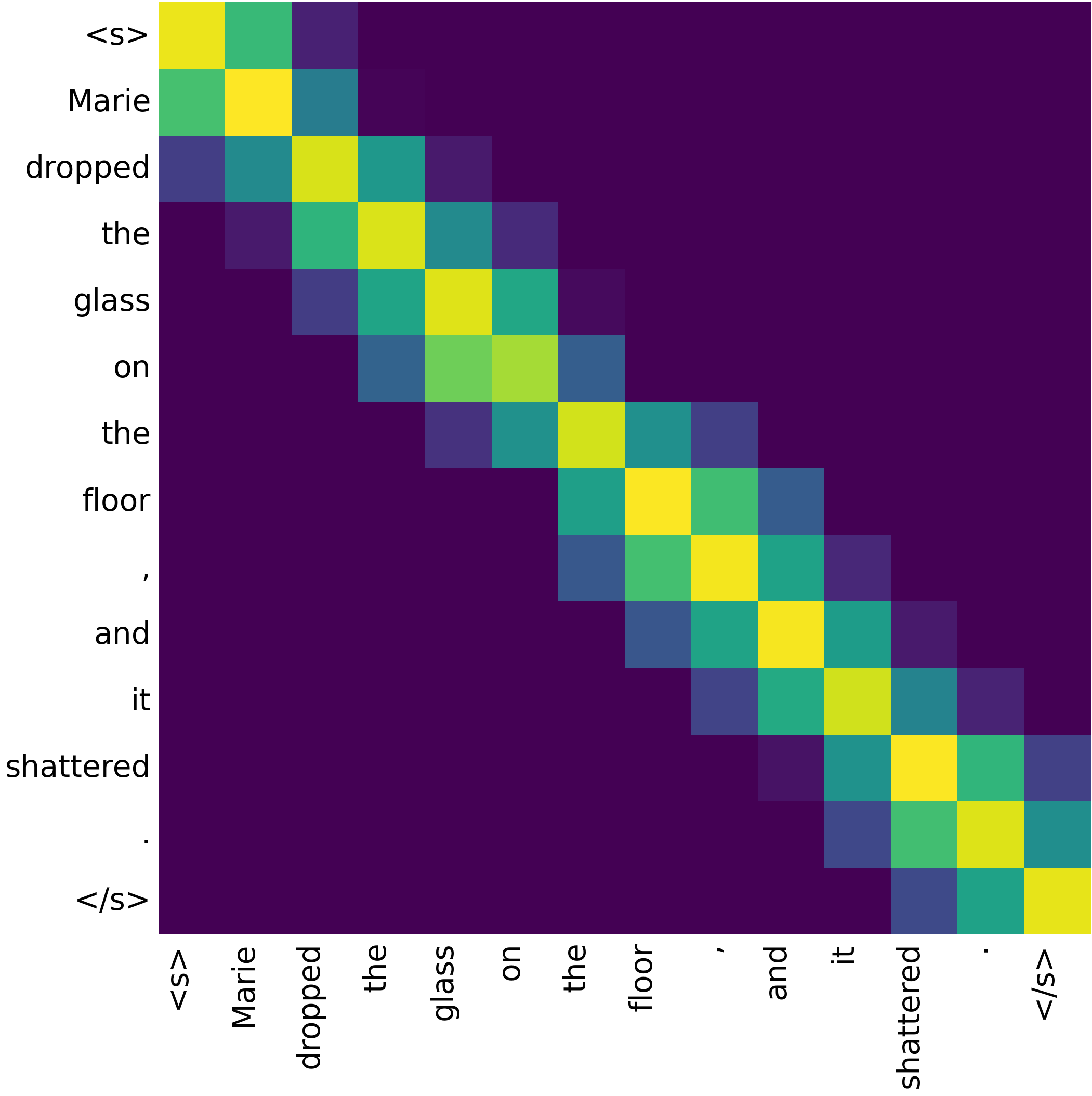}
        \caption{Local banded attention.}
        \label{fig:attn_ex1}
    \end{subfigure}\hfill
    \begin{subfigure}[t]{0.49\linewidth}
        \centering
        \includegraphics[width=\linewidth]{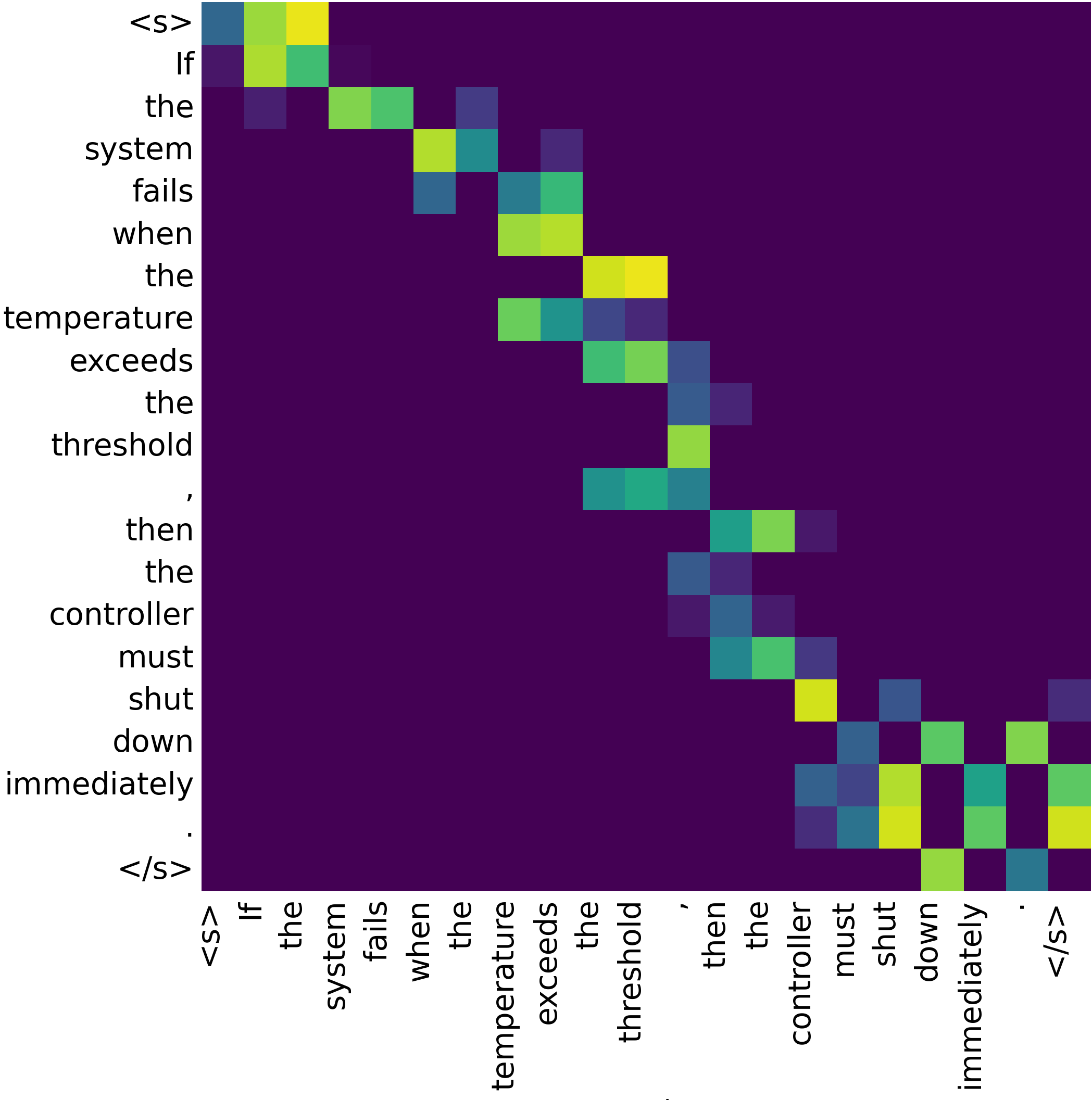}
        \caption{Structured sparse attention.}
        \label{fig:attn_ex2}
    \end{subfigure}

    \vspace{0.5em}

    \begin{subfigure}[t]{0.49\linewidth}
        \centering
        \includegraphics[width=\linewidth]{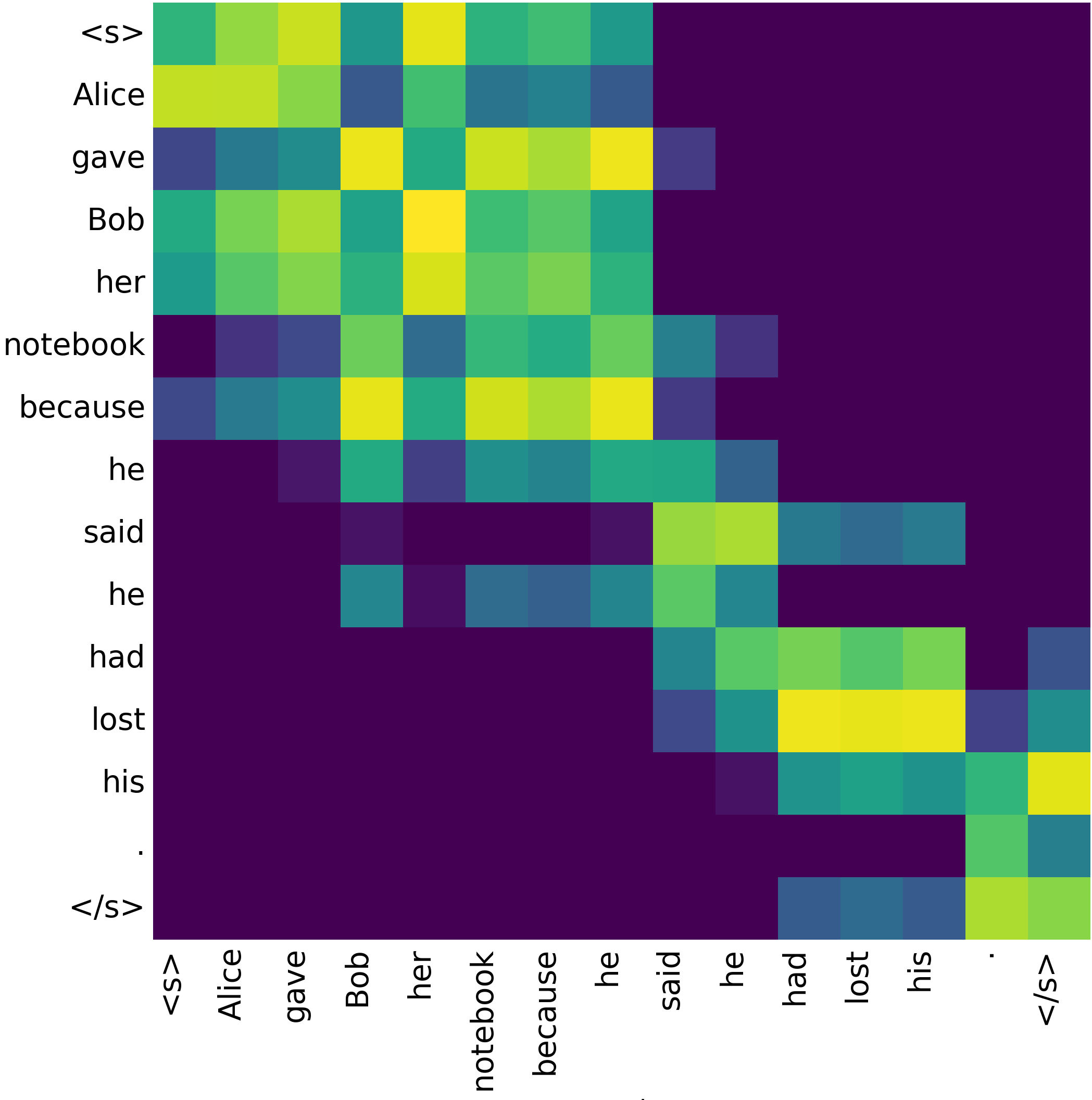}
        \caption{Block-structured attention.}
        \label{fig:attn_ex3}
    \end{subfigure}\hfill
    \begin{subfigure}[t]{0.49\linewidth}
        \centering
        \includegraphics[width=\linewidth]{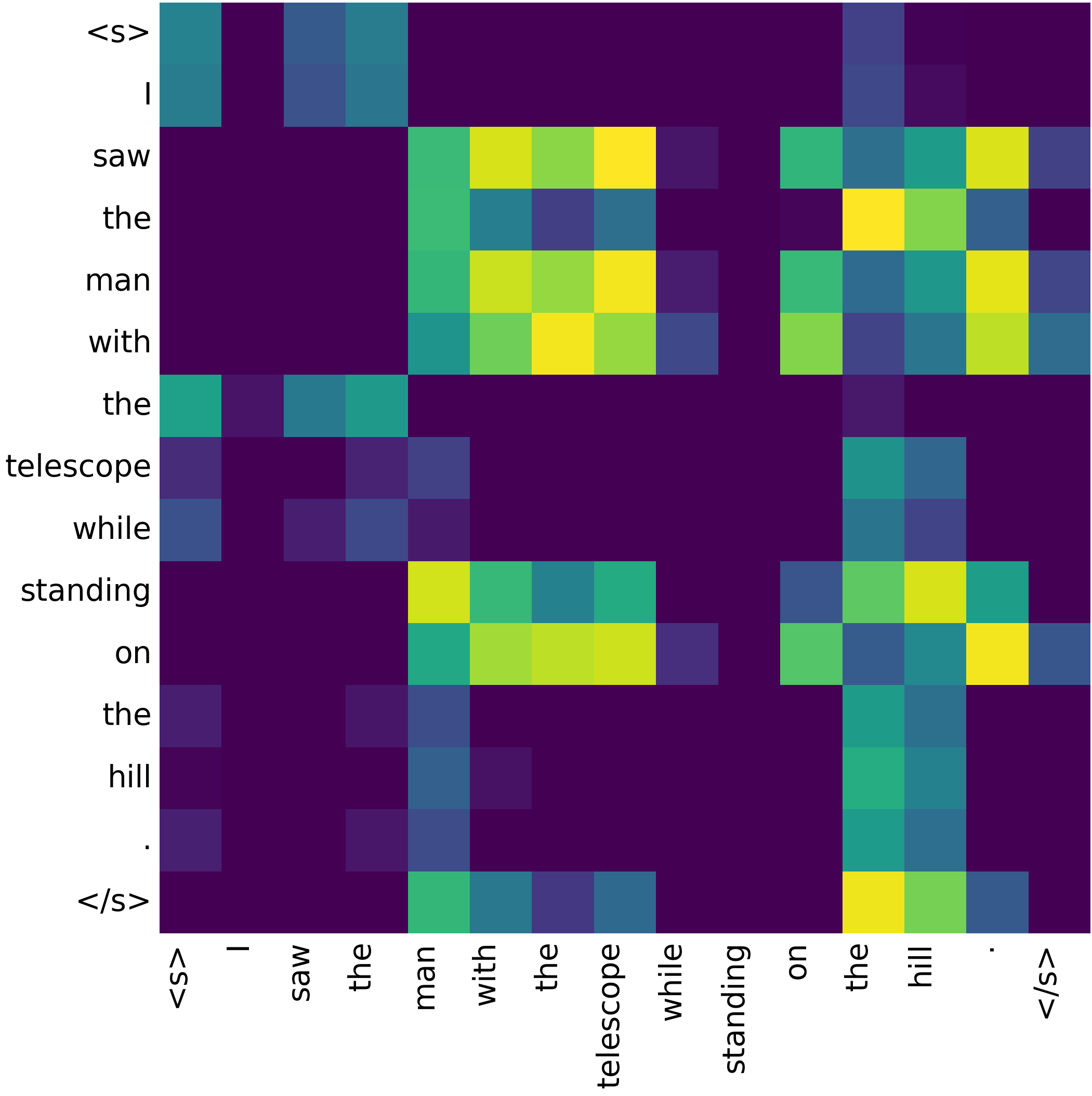}
        \caption{Multi-block attention.}
        \label{fig:attn_ex4}
    \end{subfigure}

    \caption{
    Representative attention maps produced by sliced ReLU-bump attention on natural language inputs, for a trained model (see \textsection \ref{subsub:mlm}).
    Each panel shows the attention weights of a single head and layer, illustrating distinct structural regimes.
    }
    \label{fig:relubump_attention_examples}
\end{figure}
\\
\textbf{Other kernel functions: }
A large class of piecewise linear functions can be expressed as a linear combination of ReLU functions. Although the ReLU attention is expressive enough to reproduce the ReLU bump or other such functions, the training dynamic might be facilitated in some situations by using a pre-defined shape of the kernel function. This is observed for the ReLU-bump, yet other kernel functions might be explored.


\section{Expressivity of ReLU transformers and ReLU attention}\label{sec:theory_expressivity}

In this section, we study the expressivity of the ReLU attention model. We first show that ReLU attention preserves the sequence-to-sequence expressive power of standard softmax Transformers, recovering the universal approximation result of \cite{yun2020transformersuniversalapproximatorssequencetosequence}. We then establish a second, more measure-theoretic result, showing that ReLU attention also satisfies contextual universal approximation properties over spaces of probability measures, following the framework of \cite{furuya2024transformersuniversalincontextlearners}.

\subsection{Expressivity of ReLU attention as sequence-to-sequence mapping}
In this paragraph, we prove that an attention-only architecture with ReLU attention is at least as contextually expressive as standard softmax attention, in the direction of \cite{yun2020transformersuniversalapproximatorssequencetosequence,geshkovski2024measuretomeasureinterpolationusingtransformers}. In contrast with the previous section, we focus on an architecture that contains \emph{only} attention layers, without MLP blocks.
 Given a list of $p$ source sequences and a list of $p$ target sequences, does there exist a discrete sequence of attention layers that maps each source sequence to the corresponding target sequence? Our goal is to study the expressivity of our new model of attention without using the MLP layers, parts of the standard transformer blocks. We start with our definition of contextual expressivity:
\begin{definition}[Contextual expressivity of attention]
   Given $p$ sequences $\mathbf x_i = (x_{k,i})_{k = 1,\ldots,n} \in \mathbb R^{dn}$ and $\mathbf y_i = (y_{k,i})_{k = 1,\ldots,n} \in \mathbb R^{dn}$ for $i= 1,\ldots,p$, the attention is transitive on groups of sequences if there exists a finite number $l$ of steps of attention maps $\mathcal{A}_k$ such that $\mathcal{A}_1 \circ \ldots \circ \mathcal{A}_l (\mathbf x_i )= \mathbf y_i$ for all $i=1,\ldots,p$.
\end{definition}

The ReLU kernel is as contextually expressive as the standard attention for this criterion: 
\begin{theorem}[ReLU attention is contextually expressive]\label{th:relu_att_seqtoseq}
    Let $d\geq 2$. Assume that within each sequence the $n$ tokens are pairwise distinct. Then ReLU attention is transitive on groups of sequences: for any $p$ source sequences and $p$ target sequences of length $n$ in $\R^d$, there exists a composition of at most $2p(n+1)-1$ ReLU–attention layers that maps each source sequence to its prescribed target without mixing different groups.
\end{theorem}

Note that our result is stronger than the expressivity result in \cite{geshkovski2024measuretomeasureinterpolationusingtransformers} since it has been proven that classical attention is transitive on groups of \emph{measures} instead of groups of sequences. Indeed, we can match labeled points, whereas \cite{geshkovski2024measuretomeasureinterpolationusingtransformers} matches $\mathbf x_i$ to $\mathbf y_i$ up to a permutation (depending on $i$) of the $n$ points. Below we present our proof strategy, and refer to Appendix \ref{app:proof_th_relu_att_seqtoseq} for a detailed version.
The key step for proving this result is the following disentanglement lemma, similar to \cite{yun2020transformersuniversalapproximatorssequencetosequence}, and proven in Appendix  \ref{app:proof_th_relu_att_seqtoseq}.
\begin{lemma}[Disentanglement in 1D]\label{lemma:disentanglement_1d}
 Consider $p$ sets of $n$ distinct points $\mathbf x_i$ in  dimension $1$. There exist $2p-1$ attention steps and $p$ disjoint intervals $I_i$ such that after the composition of the $2p-1$ ReLU attention layers, one has $x_{i,k} \in I_i$.
\end{lemma}

\begin{proof}[Sketch of proof of Theorem \ref{th:relu_att_seqtoseq}]
    The proof proceeds in three steps and relies on the fact that ReLU attention can act locally and selectively on individual tokens.
    
    First, we show that, in dimension one, ReLU attention can separate any family of sequences into disjoint intervals. A "splitting lemma" \ref{lemma:1D_att_split_left} allows us to iteratively disentangle the $p$ source sequences so that they occupy disjoint regions on a chosen line. This fact is proven in Lemma \ref{lemma:disentanglement_1d}.

    Second, for higher dimensional tokens, we pick a direction $\R \eta$, and use the previous result to disentangle the sequence along this specific line.
    
    Finally, once the sequences are separated, we use localised ReLU "bumps" (see Lemma \ref{lemma:relu_bump}) to adjust their orthogonal components independently: each token can be moved in the hyperplane $\eta^\perp$, without affecting their position on the line $\R \eta$. We then correct the remaining coordinates along the separation direction $\R \eta$, using once again specific ReLU updates.

    Each step can be implemented by a small number of attention layers, yielding a total bound of at most $2p(n+1) - 1$ layers.
\end{proof}
 Although the last two phases of the construction (matching orthogonal components and matching along $\R \eta)$ can be done in parallel inside a single wide attention head, we adopt a sequential construction to keep the per-layer width independent of $n$ and $p$. This yields a deeper but width-efficient architecture.

\subsection{Measure-theoretic expressivity and universality of sliced ReLU attention}\label{subsec:measure_approx_transfo}

We now show that ReLU attention also retains the measure-theoretic universal approximation properties established for softmax-based Transformers in the mean-field limit. This viewpoint follows the framework of contextual mappings introduced in \cite{furuya2024transformersuniversalincontextlearners}.

\subsubsection{Measure-theoretic attention}

To study expressivity independently of the number of tokens, we adopt the mean-field formulation of \cite{furuya2024transformersuniversalincontextlearners}. Given a probability measure $\mu$ on a domain $\Omega \subset \R^d$, an attention layer with parameters $(Q,K,V)$ defines a contextual update:
\begin{small}
\begin{equation}\label{eq:mean_field_attention}
    \Gamma_{\theta}(x,\mu) \coloneqq x + \sum_{h=1}^H W^h \int \frac{ S(Q^h x, K^h y)}{\int N(Q^hx,K^hz)d\mu(z)} V^h y d\mu(y) \,,
\end{equation}
\end{small}
where $S$ is a similarity or relation function, and $N$ is a normalization function. For an empirical measure $\mu = \frac{1}{n}\sum_i \delta_{x_i}$, this recovers the standard multi-head attention acting on token sequences. 
 In the standard softmax case, the similarity function is $S^{\text{sftm}}(Qx,Ky) \coloneqq \exp{\left(\frac{\langle Qx, Ky \rangle}{\sqrt{d}}\right)}$ and $N^{\text{sftm}} = S^{\text{sftm}}$. For our sliced ReLU attention, $S_{\Pi}^{\text{ReLU}}(Qx,Ky) \coloneqq \operatorname{ReLU}(\Pi Qx - \Pi Ky)$ and $N_{\Pi}^{\text{ReLU}}(Qx,Ky) = |\Pi Qx - \Pi Ky|$.
A Transformer layer consists of an attention update $\Gamma_\theta$ followed by a pointwise MLP $F_\xi$. Therefore, a depth-L model induces the contextual mapping
\begin{equation}
    F_{\xi_L} \diamond \Gamma_{\theta_L} \diamond \cdots \diamond F_{\xi_1} \diamond \Gamma_{\theta_1} : \Omega \times \mathcal{P}(\Omega) \to \R  \,.
\end{equation}

A more detailed derivation of this formulation from standard token-based multi-head attention is provided in Appendix \ref{app:background_mean_field_transformers}.

\subsubsection{Universality of the ReLU transformers.}

To characterize expressivity in a way that does not depend on the number of tokens, we adopt the mean-field viewpoint described above. In this setting, expressivity amounts to universal approximation over the space of probability measures. We show that the universality result in Th. 1 in \cite{furuya2024transformersuniversalincontextlearners} still holds in the case of our sliced $\operatorname{ReLU}$ attention model:
\begin{proposition}\label{prop:transf_universality_measure}
    Let $\Omega \subset \R^d$ be a compact domain and $\Lambda^* : \Omega \times \mathcal{P}(\Omega) \mapsto \R^{d'}$ be a continuous function for the $l_2 \times$ weak$^*$ topology. Then, for all $\varepsilon > 0$, there exists a depth $L$ and parameters $(\xi_l,\theta_l)_{1 \leq l \leq L}$ such that $\forall x,\mu \in  \Omega \times \mathcal{P}(\Omega)$,
    \begin{equation}
         | F_{\xi_L} \diamond \Gamma_{\theta_L} \diamond \ldots \diamond F_{\xi_1} \diamond \Gamma_{\theta_1}(x,\mu) - \Lambda^*(x,\mu)| \leq \varepsilon \,,
    \end{equation}
    with $d_{in}(\theta_l) \leq d+3d'$, $d_{head}(\theta_l) = 1$, $H(\theta_l) = d'$.
\end{proposition}
Following the approach of \cite{furuya2024transformersuniversalincontextlearners}, the proof proceeds by identifying a family of simple one-dimensional attention–MLP compositions that generate a dense algebra of contextual mappings. The key idea is that a single ReLU-based attention head, combined with an affine MLP, can implement elementary functions of the form
\begin{equation}
    \gamma_{\lambda}(x,\mu) \coloneqq \langle x,a \rangle + b + \int v\operatorname{ReLU}(\langle x,a \rangle - \langle y,a \rangle + c)d\mu(y) \,,
\end{equation}
which act as building blocks for more complex mappings. Finite sums and pointwise products of these elementary maps generate an algebra that is dense in the space of continuous functions on $\Omega \times \mathcal{P}(\Omega) $. Once this density is established (see Appendix \ref{AppendixFirstProof}), the universal approximation result follows exactly as in \cite{furuya2024transformersuniversalincontextlearners}, with only the base algebra replaced.

 In the same way as in \cite{furuya2024transformersuniversalincontextlearners}, the approximation is independent of the number of tokens $n$, and the number of heads and hidden dimension do not depend upon $\varepsilon$ (only the depth does). 
 The key point of the proof is the algebra structure, which can be approximated by the MLPs.

\begin{table*}[t]
\centering
\begin{tabular}{l|ccccc|c}
\toprule
\textbf{Model} & \textbf{ListOps} & \textbf{Text} & \textbf{Retrieval} & \textbf{Image} & \textbf{Pathfinder} & \textbf{Avg} \\
\midrule
Softmax       & \textbf{39.3} & \textbf{66.8} & 76.5 & \textbf{44.2} & 72.2 & 59.8 \\
Sliced ReLU   & 37.3 & 64.3 & \textbf{79.5} & 43.8 & \textbf{89.9} & \textbf{62.9} \\
\bottomrule
\end{tabular}
\caption{Comparison between softmax attention and sliced ReLU attention on the LRA benchmark (accuracy, higher is better).}
\label{tab:lra_results_alt}
\end{table*}

\begin{table*}[t]
\centering
\footnotesize
\setlength{\tabcolsep}{4pt}
\begin{tabular}{l|cc|cc|cc|cc|cc}
\toprule
& \multicolumn{2}{c|}{\textbf{ListOps (2K)}}
& \multicolumn{2}{c|}{\textbf{Text (4K)}}
& \multicolumn{2}{c|}{\textbf{Retrieval (2×4K)}}
& \multicolumn{2}{c|}{\textbf{Image (1K)}}
& \multicolumn{2}{c}{\textbf{Pathfinder (1K)}} \\
\textbf{Model}
& \textit{Inf.} & \textit{Train}
& \textit{Inf.} & \textit{Train}
& \textit{Inf.} & \textit{Train}
& \textit{Inf.} & \textit{Train}
& \textit{Inf.} & \textit{Train} \\
\midrule
Softmax
& 140.1 & 61.5
& 30.2 & 13.6
& 20.8 & 7.25
& \textbf{552.4} & \textbf{202.9}
& 678.8 & 223.2 \\

Sliced ReLU
& \textbf{202.2} & \textbf{76.7}
& \textbf{120.4} & \textbf{34.2}
& \textbf{55.8} & \textbf{23.0}
& 372.1 & 153.2
& \textbf{716.2} & \textbf{301.1} \\
\bottomrule
\end{tabular}
\caption{Per-task throughput (samples/s per GPU, higher is better) for softmax and sliced ReLU attention on LRA, measured on 16GB V100 GPUs. To match the same effective batch size per update, the softmax baseline required gradient accumulation on long-sequence tasks (e.g., ListOps, Retrieval, Pathfinder), often resulting in very small per-GPU batches, while sliced ReLU generally fits larger batches with little or no accumulation. Throughput is normalized per GPU and reported as measured.}
\label{tab:lra_speed}
\end{table*}

\section{Experiments.}
\label{SecExperiments}
While our contribution is primarily theoretical, we conduct small-scale experiments with sliced ReLU attention due to computational constraints. These preliminary results suggest that this model merits further exploration: the ReLU-bump exhibits comparable performances with the softmax in several tasks, while the long-range interaction given by sliced ReLU may be beneficial for different tasks.

\subsection{Long Range Arena benchmark}

To test the behavior of our model on long-context tasks, we conducted experiments on the Long Range Arena (LRA) benchmark \cite{tay2020longrangearenabenchmark}, a list of tasks specifically designed to test the robustness and efficiency of Transformer variants under long input sequences. Since our architecture is implemented in PyTorch, we did not use the official LRA codebase, but instead reimplemented all tasks from the publicly available archived version of the benchmark. For consistency, we reproduced the preprocessing pipelines and task formulations, while using our own training setup and architectural components, such as pre-layer normalization \cite{xiong2020layernormalizationtransformerarchitecture}, RoPE embeddings \cite{su2023roformerenhancedtransformerrotary} in some configurations, and model sizes differing from those in the original benchmarks. Both the standard softmax model and our sliced ReLU attention model were reimplemented within the same framework to ensure a fair comparison. We evaluate on the standard LRA tasks, including LongListops, byte-level text classification, byte-level document retrieval (AAN), pixel-level image classification on the CIFAR-10 dataset\cite{Krizhevsky2009LearningML}, and the Pathfinder task. All tasks used mean-pooling. For positional embeddings, we used RoPE for ListOps, text classification, and retrieval, and a learned positional embedding for pixel-level image classification and Pathfinder. Results are presented in Table \ref{tab:lra_results_alt}, with training and inference speed in Table \ref{tab:lra_speed}. Hyperparameters are written in Table \ref{tab:lra_hparams_tasks}. All trainings were done on two 16Gb V100 GPUs.

\begin{table}[t]
\centering
\footnotesize
\setlength{\tabcolsep}{5pt}
\begin{tabular}{r|ccc}
\toprule
\textbf{Model} & \textbf{Softmax} & \textbf{ReLU-bump} & \textbf{ReLU} \\
\midrule
ModelNet40 (Acc.~\%) & \textbf{86.3} & 85.4 & 76.2 \\
\bottomrule
\end{tabular}
\caption{Classification accuracy on ModelNet40 (higher is better).}
\label{tab:modelnet_results}
\end{table}

\begin{table*}[t]
\centering
\small
\setlength{\tabcolsep}{4pt}
\begin{tabular}{l|c|ccccc}
\toprule
\textbf{Model} 
& \textbf{MLM} 
& \textbf{STS-B} 
& \textbf{QQP} 
& \textbf{MNLI-m} 
& \textbf{QNLI} 
& \textbf{SST-2} \\
\midrule
Softmax 
& 1.29 
& 87.2 
& 89.8 
& 86.3 
& 91.3 
& 92.9 \\
Sliced ReLU-bump 
& 1.63 
& 81.2 
& 87.7 
& 74.3 
& 80.7 
& 84.5 \\
\bottomrule
\end{tabular}
\caption{Masked language model (MLM) evaluation loss (cross-entropy) after 75k pretraining steps, followed by downstream GLUE performance. STS-B reports Spearman correlation; other tasks report accuracy.}
\label{tab:mlm_glue}
\end{table*}

\subsection{Tiny ViT experiment}

We evaluate the feasibility of the proposed sliced ReLU attention on a low-capacity Vision Transformer (ViT) trained on CIFAR-10 and Tiny Imagenet. A tiny ViT architecture with a CNN-based patch embedding is used for all three variants:\footnote{Standard implementations of linear attention failed to match the performance of the other methods.} S tandard softmax attention, sliced ReLU (with linear or nonlinear projections), and sliced ReLU bump:
The sliced ReLU and sliced ReLU bump models contain approximately $4.54$M parameters, similar to the softmax baseline. All training conditions, including a patch size of $4$, optimizer (AdamW), data augmentation (mixup and cutmix) are kept identical across models. 
For each setup, we selected the best hyperparameters and showed the median over four runs.
The evolution of the validation accuracy is shown in Fig.~\ref{FigAccuracyPlotViT}.
\begin{figure}[ht!]
    \centering
    \begin{subfigure}[t]{0.50\textwidth}
        \centering
        \includegraphics[width=\linewidth]{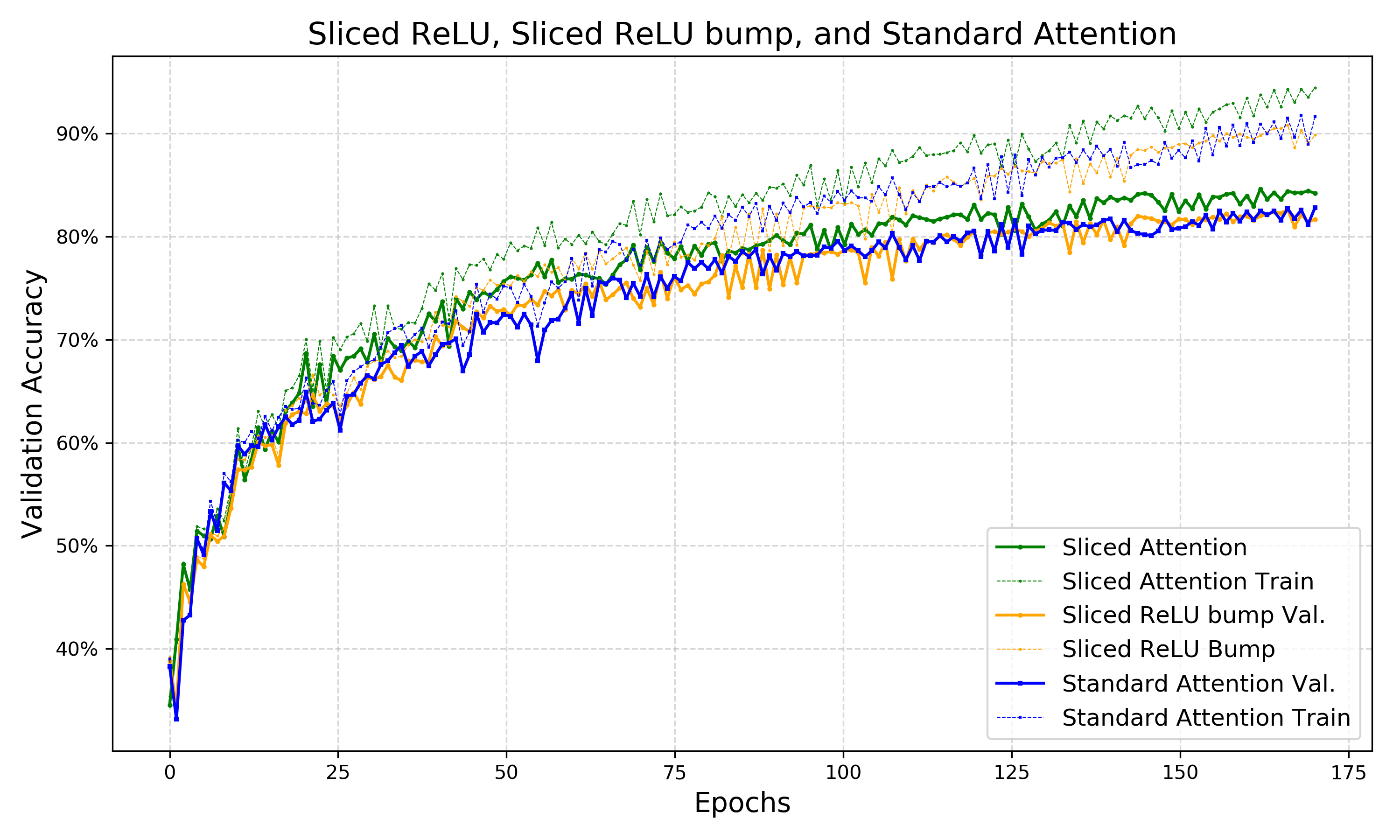}
        \caption{On CIFAR-10, sliced ReLU bump and softmax attention perform equally, and are slightly outperformed by sliced ReLU.}
        \label{fig:subfig1}
    \end{subfigure}
    \hfill
    \begin{subfigure}[t]{0.50\textwidth}
        \centering
        \includegraphics[width=\linewidth]{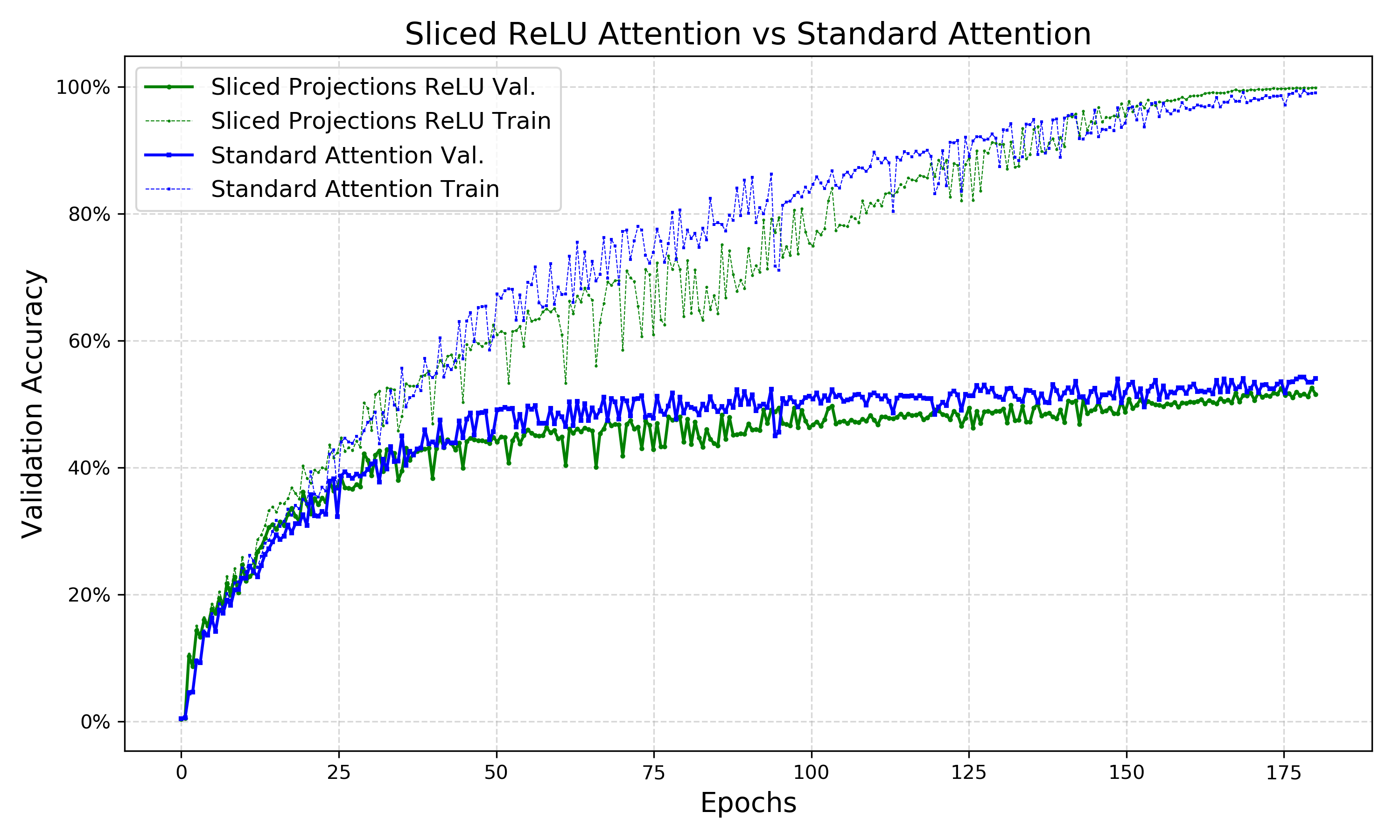}
        \caption{On Tiny Imagenet, standard attention outperforms sliced attention.}
        \label{fig:subfig2}
    \end{subfigure}

    \caption{Accuracy plots for the different attention models. Each curve represents two training runs.}
    \label{FigAccuracyPlotViT}
\end{figure}
As illustrated in Fig.~\ref{fig:subfig1}, sliced ReLU attention consistently outperforms the sliced ReLU-bump variant. On CIFAR-10, sliced attention slightly outperforms standard attention, suggesting improved inductive bias and regularization. On Tiny ImageNet, standard attention exhibits better generalization than sliced attention, achieving higher test performance while both methods fit the training data equally well. This observation suggests that incorporating stronger regularization may improve sliced attention in future work.
Other ablation experiments, including slicing the standard softmax attention, are included in Appendix \ref{AppendixTinyViT}. 

\subsection{ModelNet40 classification, PointCloud Transformers}
To assess the geometric behavior of our attention mechanism, we evaluate it on 3D point-cloud classification using the ModelNet40 dataset \cite{Wu2015}, which contains 40 object categories represented as raw point sets. We follow the Point Cloud Transformer framework \cite{Guo2021PCT} but intentionally adopt a minimal backbone: we remove neighborhood grouping, farthest-point sampling, and hierarchical blocks, retaining only the core point-wise Transformer layers. This setup isolates the effect of the attention kernel itself, without additional geometric machinery. We apply standard data augmentations (anisotropic scaling, random translations, in-plane rotations, point dropout) and compare softmax attention, our ReLU-bump kernel, and the simpler sliced ReLU kernel. In this controlled setting, the ReLU-bump kernel consistently outperforms the plain ReLU variant and closely matches softmax attention. This indicates that the bump formulation captures localized geometric structure more effectively, while still benefiting from the global and efficient sliced construction. Results are reported in Table \ref{tab:modelnet_results}.

\subsection{Bidirectional Language Model}\label{subsub:mlm}

To further evaluate our approach on large-scale natural language processing tasks, we consider a bidirectional Transformer language model pretrained using masked language modeling.
Our goal is not to optimize for state-of-the-art performance, but to assess whether sliced ReLU-bump attention supports standard large-scale pretraining and fine-tuning procedures on non-trivial NLP tasks.
We evaluate the resulting models on downstream tasks from a subset of the GLUE benchmark.

\subsubsection{Masked Language Modeling Pretraining}

We pretrain an encoder-only Transformer model using the masked language modeling (MLM) objective introduced in \cite{devlin2019bertpretrainingdeepbidirectional}, where the model is trained to predict randomly masked tokens from their surrounding context.
To evaluate the impact of the attention mechanism, we train two models with identical architectures, differing only in the attention formulation: one using standard softmax attention, and the other using our sliced ReLU-bump attention.
\\
\textbf{Model architecture:} 
Following the recommendations of \cite{izsak2021trainbertacademicbudget}, we train large Transformer models with 24 layers, a hidden dimension of 1024, and an intermediate dimension of 4096.
The softmax attention variant uses 16 attention heads, while the sliced ReLU-bump variant uses 64 heads, as selected based on a preliminary ablation study.
We use rotary positional embeddings (RoPE) \cite{su2023roformerenhancedtransformerrotary}, following \cite{warner2024smarterbetterfasterlonger}.
All models adopt a pre-normalization Transformer architecture \cite{xiong2020layernormalizationtransformerarchitecture}.
\\
\textbf{Dataset: }
Pretraining is performed on a 160\,GB English corpus composed of a mixture of FineWeb \cite{penedo2024finewebdatasetsdecantingweb}, Wikipedia, and BookCorpus.
\\
\textbf{Training setup: }
Models are trained for 75k optimization steps with a global batch size of 1024 and a sequence length of 512 tokens.
We use the byte-level RoBERTa tokenizer \cite{liu2019robertarobustlyoptimizedbert} with a masking ratio of 20\%. The masked language modeling objective is implemented using a standard cross-entropy loss over the masked tokens.
Optimization is performed using AdamW \cite{loshchilov2019decoupledweightdecayregularization} with parameters $\beta_1 = 0.9$, $\beta_2 = 0.98$, $\varepsilon = 1e-6$, and a weight decay of 0.01. We use a cosine decay scheduler, with a peak learning rate of 5e-4. We use a dropout rate of 0.1 throughout the model, and an attention dropout of 0.05.

In Table~\ref{tab:mlm_glue}, we report the final masked language modeling evaluation loss for both attention variants.
As expected, the softmax attention model achieves a lower loss.
Nevertheless, the sliced ReLU-bump model reaches a non-trivial MLM loss, indicating successful large-scale pretraining under identical optimization and data settings. The observed losses are consistent with those reported for reasonably trained BERT-style models under limited training budgets \cite{izsak2021trainbertacademicbudget}.

In Figure~\ref{fig:relubump_attention_examples}, we illustrate representative attention patterns produced by the trained model. Additional visualizations are provided in Appendix~\ref{app:visu_band}.

\subsubsection{Downstream Evaluation on GLUE}

We test the performances of our pretrained models on the GLUE benchmark \cite{gluedataset_2018}, covering sentence similarity, paraphrase detection, sentiment classification, and entailment. Fine-tuning is performed independently for each task using standard task-specific heads. For each task and each attention variant, we perform a small grid search across batch size and learning rate and report the best result. Table~\ref{tab:mlm_glue} reports downstream performance for both attention variants.
Standard softmax attention achieves higher performance across all evaluated tasks, with performance gaps that vary across benchmarks and remain smaller on semantic similarity and paraphrase detection. Despite this gap, the sliced ReLU-bump model attains consistent and non-trivial downstream performance across all tasks, indicating that it learns transferable linguistic representations and supports standard fine-tuning procedures following large-scale masked language model pretraining. Evaluating sliced ReLU attention on substantially longer-context tasks, where its quasi-linear complexity provides clearer computational benefits, is left for future work.


\section*{Impact Statement}
By reducing the cost of long-range attention, our method can decrease energy consumption associated with training and inference, contributing to more sustainable ML practices. The societal impact of enabling larger models depends on their specific uses and governance.

\bibliography{biblio}
\bibliographystyle{icml2026}

\newpage
\appendix
\onecolumn

\section{ReLU as an asymmetric CND kernel}\label{app:cpd_relu}

For completeness, we provide here the formal statement and proof that the 
kernel $(x,y) \mapsto \operatorname{ReLU}(x-y)$ is conditionally positive definite
of order 1 on $\mathbb{R}$. This property is used in Section~\ref{sec:model}
to justify the value-centering strategy for sliced ReLU attention.

\begin{proposition}
    Let $x_1, \ldots, x_n \in \R$ be $n$ distinct points and $\gamma_1, \ldots, \gamma_n \in \R$ such that $\sum_i \gamma_i = 0$.
    Then:
    \begin{equation}
        -\sum_{1\leq i,j \leq n } \gamma_i \gamma_j \operatorname{ReLU}(x_i - x_j) \geq 0 \,,
    \end{equation}
    with equality if and only if $\gamma_1 = \ldots = \gamma_n = 0$.
\end{proposition}
\begin{proof}
    Let us rewrite, using identity \eqref{eq:relu_is_ed} and the fact that $\sum_i \gamma_i = 0$:
    \begin{align}
        -\sum_{1\leq i,j \leq n } \gamma_i \gamma_j \operatorname{ReLU}(x_i - x_j) &= -\frac{1}{2}\sum_{1\leq i,j \leq n } \gamma_i \gamma_j (|x_i - x_j| + x_i - x_j) \\
        &= -\frac{1}{2}\sum_{1\leq i,j \leq n } \gamma_i \gamma_j |x_i - x_j| \,.
    \end{align}
    The fact that the Energy-Distance kernel $(x,y) \mapsto -|x-y|$ is conditionally positive definite in $\R$ is well-known, see 8.4 in \cite{Wendland_2004}, ending the proof.
\end{proof}

\section{Proof of Theorem \ref{th:relu_att_seqtoseq}}\label{app:proof_th_relu_att_seqtoseq}

To prove the result, we first prove a lemma in dimension $1$ which shows that ReLU attention can separate the groups of points.
\begin{lemma}[Splitting lemma in 1D]\label{lemma:1D_att_split_left}
    Let $S = \{ \mathbf{x_1}, \ldots, \mathbf{x_p}\}$ be a set of distinct sequences $\mathbf{x} = (x_1 , x_2, \ldots , x_n) \in \R^n$, and $K$ be a compact set such that $\underset{\mathbf{x \in S}}{\max}\max_i x_i < \min K$. Then there exist two subsets $S_1, S_2 \subset S$ and a 1D attention layer $\mathcal{A}$ such that $\mathcal{A} \equiv \Id$ on $K$ and such that 
    \begin{equation}
        \max \mathcal{A}(S_1) < \min \mathcal{A}(S_2) < \max \mathcal{A}(S_2) < \min K \,,
    \end{equation}
    where for $T\subset \R^n$ we set
    $\min T:=\min_{x\in T}\min_m x_m$ and $\max T:=\max_{x\in T}\max_m x_m$.
\end{lemma}

\begin{proof}
    Write $S = \{ x_1, \ldots, x_p\}$. Since we match points only up to permutation in 1D, we can assume that the entries in $\mathbf{x_i} = (x_{i,1}, \ldots, x_{i,n})$ are ordered increasingly.
Now, we can use the lexicographic order between sequences. 
Given two vectors $\mathbf{x}, \mathbf{y} \in \mathbb{R}^n$ where:
\[
\mathbf{x} = (x_1 , x_2, \ldots , x_n), \quad 
\mathbf{y} = (y_1, y_2, \ldots, y_n),
\]
both in increasing order, meaning that $x_i \leq x_{i+1}$ and $y_i \leq y_{i+1}$,
we define the \textbf{lexicographic order} $<_{\text{lex}}$ as:

\[
\mathbf{x} <_{\text{lex}} \mathbf{y} \quad \text{if and only if} \quad 
\exists \, k \in \{1, \ldots, n\} \text{ such that:}
\]
\begin{itemize}
    \item $x_i = y_i$ for all $i < k$, \textbf{and}
    \item $x_k < y_k$.
\end{itemize}
Without loss of generality, we assume $\mathbf x_1 \prec_{\mathrm{lex}} \ldots \prec_{\mathrm{lex}} \mathbf x_p $. Denote by $l \in \llbracket 1, n \rrbracket$ the sequence index at which they differ, i.e., there exists $j_0 \in \llbracket 1, p \rrbracket$ such that 
    \begin{itemize}
        \item $\forall i \in \llbracket 1, p \rrbracket, \forall m < l, x_{i,m} = x_{1,m}$\,,
        \item $x_{1,l} = x_{2,l} = \ldots = x_{j_0,l} < x_{j_0 + 1,l} \leq x_{j_0 + 2,l} \leq \ldots x_{p,l}$\,.
    \end{itemize}
    Define the sequence sets $S_1 \coloneqq \{ x_1, \ldots , x_{j_0} \}$ and $S_2 \coloneqq \{ x_{j_0 + 1}, \ldots , x_{p} \}$ and the thresholds $t_1 \coloneqq x_{1,l}$ and $t_2 \coloneqq \min \{x_{i,k} \, | \,   x_{i,k} > t_1\}$. Pick $L, R \in \R$ such that $L < \min S$ and $R > \max K$, and $a$ the increasing affine map such that $a(t_1) = L$ and $a(t_2) = R$, and a constant value $v \in \R$. Consider the attention mapping 
    \begin{equation}
        \mathcal{A}(x_i; x_{i,k}) \coloneqq x_{i,k} + \sum_{m=1}^n \operatorname{ReLU}(x_{i,k} - a(x_{i,m}))v \,.
    \end{equation}
    Then, with such a choice for $a$, this layer leaves $K$ untouched if $R > \max K$, since then all the ReLU terms vanish on K. On $S_1$ and $S_2$ it writes as
    \begin{equation}
        \mathcal{A}(x_i; x_{i,k}) = \begin{cases}
            x_{i,k} + \sum_{m=1}^l \operatorname{ReLU}(x_{i,k} - a(x_{i,m}))v \hspace{5mm} \text{if } i \leq j_0 \,, \\
            x_{i,k} + \sum_{m=1}^{l-1} \operatorname{ReLU}(x_{i,k} - a(x_{i,m}))v \hspace{5mm} \text{if } i > j_0 \,.
        \end{cases}
    \end{equation}
    Remark that, since the first $l-1$ terms are the same for each sequence, and the $l$-th term is the same for all sequences in $S_1$, the mappings are affine functions
    \begin{equation}
        \mathcal{A}(x_i; x_{i,k}) = \begin{cases}
            \alpha_1 x_{i,k} + \beta_1 \hspace{5mm} \text{if } i \leq j_0 \,, \\
            \alpha_2 x_{i,k} + \beta_2 \hspace{5mm} \text{if } i > j_0 \,,
        \end{cases}
    \end{equation}
    where $\alpha_1 = 1 + lv, \beta_1 = - v\sum_{m=1}^l a(x_{1,m})$, and $\alpha_2 = \alpha_1 - v, \beta_2 = \beta_1 + va(x_{1,l}) = \beta_1 + vL$. Define $M_-< M_+$ as the minimum and maximum entries of $S$. Then, we need to satisfy the inequalities
    \begin{equation}
        \begin{cases}
            0 \leq 1 + lv \leq 1 \hspace{5mm} \text{(to keep the original ordering of $M_-, M_+$ and $K$ and to control growth)} \,, \\
            \alpha_1 M_+ + \beta_1 < \alpha_2 M_- + \beta_2 \hspace{5mm} \text{(separation of $S_1$ and $S_2$)}\,, \\
            \alpha_2 M_+ + \beta_2 \leq M_+ \hspace{5mm} \text{(to ensure sequences move left)}\,.
        \end{cases}
    \end{equation}
    We will investigate the feasibility of these three inequalities, choosing an appropriate $L$. The first inequality is equivalent to 
    \begin{equation}
        - \frac{1}{l} \leq v \leq 0 \,.
    \end{equation}
    The second one, from the definition of $\alpha_1, \alpha_2, \beta_1, \beta_2$, is implied by the stronger inequality
    \begin{equation}
        v < - \frac{M_+ - M_-}{M_+ - L}\,.
    \end{equation}
    Since $a(x_{1,m}) \leq L$ for $m<l$, the third inequality is implied, if $v$ is negative (first inequality) by
    \begin{equation}
        L \leq M_+ \,,
    \end{equation}
    which is always true by definition of $L$ and $M_+$. In order for the first and second one to be compatible we need
    \begin{equation}
        L < M_+ - l(M_+ - M_-) \,.
    \end{equation}
    Choose $L$ sufficiently negative, then choose $v$ accordingly in the non-empty interval $\left(-1/l, -(M_+ - M_-)/(M_+ - L)\right)$ to conclude the proof.
\end{proof}

\begin{remark}[Singleton placement in 1D]
    Note that the proof above is greatly simplified if $l = 1$ and $S_1$ has cardinal one. Indeed, in this case, $\alpha_2 = 1$ and $\beta_2 = 0$ so that $\mathcal{A}$ is the identity on $S_2$ (and $K$). On the unique sequence in $S_1$, the layer reduces to an affine map 
    \begin{equation}
        \mathcal{A}_1(x) = (1+v)x - vL \,.
    \end{equation}
    This allows us to pick parameters $v,L$ freely, sending the unique sequence in $S_1$ arbitrarily far to the left or to the right of the compact set $K$. This observation will be useful in the proof of the following lemma.
\end{remark}
Now, we can prove the result:
\begin{lemma}[Disentanglement in 1D]\label{lemma:disentanglement_1d}
 Consider $p$ sets of $n$ distinct points $\mathbf x_i$ in  dimension $1$. There exist $2p-1$ attention steps and $p$ disjoint intervals $I_i$ such that after the composition of the $2p-1$ ReLU attention layers, one has $x_{i,k} \in I_i$.
\end{lemma}

\begin{proof}
    Fix a threshold $t \in \R$ so that all the points lie strictly to its left. We prove the result by iteratively moving the sequences to the right of $t$, after separating them using Lemma \ref{lemma:1D_att_split_left}. Set $K_0 = \{t\}$ and $U = S$ the current working set. Define $F$ a family of parked closed intervals, initially empty. By parked intervals, we mean intervals containing sequences not in the working set $U$ that are not disentangled yet, and we will use iterations that leave these sequences untouched. We will maintain at each iteration the following invariant, defining $K \coloneqq K_0 \cup \cup_{I \in F} I$:
    \begin{itemize}
        \item All points of $U$ lie strictly to the left of $K$.
        \item Every layer we apply is the identity mapping on $K$, leaving it untouched.
        \item All sequences already placed lie strictly to the right of $t$, in disjoint intervals, and are included in $K_0$.
    \end{itemize}
    
    While $U \neq \emptyset$, we do the following iteration:

    First, if $U$ has cardinal one, from the remark above, we can send its unique sequence strictly to the right of $K$, using a single attention layer, in a new disjoint interval $I$. We then apply the updates: if $F \neq \emptyset$, pick the leftmost interval $J \in F$, corresponding to a set of sequences $S_J$ and $U \gets S_J$, $F \gets F \setminus\{ J \}$, and $K_0 \gets K_0 \cup I$. This preserves the invariant. If $F = \emptyset$, set $U \gets \emptyset$, $K_0 \gets K_0 \cup I$, ending the algorithm.
    
    If $U$ contains more than one sequence, apply Lemma \ref{lemma:1D_att_split_left} to the pair $(U,K = K_0 \cup \cup_{I \in F})$. We get an attention layer that is the identity on $K$, and splits $U$ into two sets $U_1$ and $U_2$, located in disjoint intervals $I_1$ and $I_2$ such that $\max I_1 < \min I_2$, strictly to the left of $K$. Then, we keep iterating on $U_1$ while leaving the rest of the sequence untouched, through the update $U \gets U_1$ and $F \gets F \cup \{I_2\}$.

    Define $B$ as the number of active blocks, i.e. $B = 1 + |F|$, where $|F|$ is the cardinal of $F$, as well as $m = |U|$ the number of sequences in the working set. Initially, $B_{\text{init}} = 1$. At each splitting step ($|U| \geq 2$), $m$ strictly decreases, since $|U_1|<|U|$, and $B$ increases by one. At each placement step, if $F \neq \emptyset$, $m$ resets to $|S_J|$ and $B$ decreases by one. If $F = \emptyset$, then $U \gets \emptyset$, the algorithm stops and $B$ goes from $1$ to $0$. Since each split strictly decreases $m$, there cannot be an infinite number of splits without a placement step, and each placement reduces the number of unplaced sequences by one, so that there are exactly $p$ placements in the algorithm. This proves the algorithm terminates, and $B_{\text{final}} = 0$.
    
    To count the number of attention layers, define $n_p$ the number of placement steps, and $n_s$ the number of splitting steps performed by the algorithm. Since the algorithm terminates, $n_p = p$, so that 
    \begin{equation}
        0 = B_{\text{final}} = B_{\text{init}} - n_p + n_s = 1 - p + n_s \,, 
    \end{equation}
    so that $n_s = p-1$, and the total number of attention layers, one per step, is equal to
    \begin{equation}
        n_s + n_p = 2p-1 \,.
    \end{equation}
\end{proof}
A final basic ingredient to prove the result is the well-known existence of a bump function in 1D that is a combination of ReLU functions:
\begin{lemma}[ReLU bump]\label{lemma:relu_bump}
    Let $x_0 \in \R$ and $\delta >0$. Then there exists a linear combination of ReLU functions that is equal to $1$ at $x_0$, and equal to $0$ outside of the interval $]x_0 - \delta, x_0 + \delta[$.
\end{lemma}
\begin{proof}
    Just define the function
    \begin{equation}
        \phi_{x_0,\delta}(x) \coloneqq \frac{1}{\delta} \left[ \operatorname{ReLU}(x-x_0 + \delta) + \operatorname{ReLU}(x-x_0 - \delta) - 2\operatorname{ReLU}(x-x_0)\right]\,,
    \end{equation}
    which is a linear combination of three ReLU functions.
\end{proof}
Note that such a function can be modeled by using three simultaneous ReLU attention heads. We can now prove our main theorem:
\begin{proof}[Proof of the theorem]
We first start with the important remark that the normalization in the attention map can be counterbalanced in all the steps of the proof using the value matrix. Therefore, we reduce the proof to the unnormalized case of attention.

    The first step of the proof is to disentangle the sequences. We pick a unit vector $\eta$ such that, within a sequence, the projection on $\R \eta$ is injective and such that, for two different sequences $x_i, x_j$, the 1D sets $\{\langle \eta, x_{i,k} \rangle \}_{1\leq k \leq n}$ and $\{\langle \eta, x_{j,k} \rangle \}_{1\leq k \leq n}$ are distinct. This is possible because the family of directions $\{\R (x_{i,k} - x_{j,l})\}$ is finite. We apply the disentanglement Lemma \ref{lemma:disentanglement_1d} with values proportional to the vector $\eta$, so that the sequences are disentangled along $\R \eta$.

    Next, we match the tokens on the components orthogonal to $\eta$. For each sequence $i$ and each position $k$, define the orthogonal residual
    \begin{equation}
        r_{i,k} \coloneqq p_{\eta^\perp}(y_{i,k}-x_{i,k}) \,,
    \end{equation}
    where $p_{\eta^\perp}$ is the orthogonal projection onto $\eta^\perp$. We also define $\delta \coloneqq \min \{ |\langle \eta, x_{i,k} - x_{j,l} \rangle| \,,\,(i,k) \neq (j,l) \} $, which is non-zero, since the $\eta$-components of the $x$ entries are separated after disentanglement. For a given $x_{i,k}$, the update 
    \begin{equation}
        x \mapsto x + r_{i,k} \phi_{\langle \eta, x_{i,k} \rangle, \delta}(\langle \eta, x \rangle)\,,
    \end{equation}
    may be realized with a single ReLU attention layer (with three heads), that moves only $x_{i,k}$, by exactly $r_{i,k} \in \eta^{\perp}$, and leaves all the other points untouched. Do this for all $x_{i,k}$ where $r_{i,k} \neq 0$. After at most $pn$ layers, all points verify
    \begin{equation}
        p_{\eta^{\perp}}(x_{i,k}) = p_{\eta^{\perp}}(y_{i,k})\,.
    \end{equation}
    
    Finally, we need to match the $\R \eta$ components. We apply (iteratively or in parallel) the ReLU updates
    \begin{equation}
        x \mapsto x + \langle \eta, y_{i,k} - x_{i,k} \rangle \phi_{\langle \eta, x_{i,k} \rangle,\delta}(\langle \eta , x \rangle)\eta\,,
    \end{equation}
    for each $x_{i,k}$.

    The first step, from Lemma \ref{lemma:disentanglement_1d}, is done in $2p-1$ layers (one head per layer), and the next two are done in $pn$ layers each (three heads per layer), so that the total layer count is equal to 
    \begin{equation}
        2p - 1 + 2pn = 2p(n+1)-1\,.
    \end{equation}
\end{proof}

\begin{remark}
    The argument could be carried with only one head per layer, without the ReLU bump. However, each layer would depend on the outcome of the previous ones, in contrast with our proof, where layers are independent from one another in the matching phase, after the initial splitting phase.
\end{remark}

\section{Background on measure-theoretic Transformers in the mean-field limit}\label{app:background_mean_field_transformers}

Here, we detail the background material required for the measure-theoretic expressivity results in \textsection \ref{subsec:measure_approx_transfo}. We briefly recall how standard multi-head attention can be interpreted as an in-context mapping on token sequences, and how this formulation extends naturally to the mean-field setting \cite{furuya2024transformersuniversalincontextlearners}. This provides the formal link between the discrete Transformer architecture and the contextual map $\Gamma_\theta$ used in our universality theorem.

\subsection{Attention as in-context mapping for a token sequence}

An attention head is a parametrized map that takes as input a token sequence $X = (x_1, \ldots, x_n) \in \R^{n \times d}$ and outputs a token sequence of the same length $\mathcal{A}_\theta(X) \in \R^{n \times d}$. Attention produces context-dependent updates by allowing each token to aggregate information from all others, in contrast with MLP mappings that treat each input in an independent way. In practice, several attention heads are combined at each layer into a multi-head attention layer with a residual connection, written as:
\begin{equation}\label{eq:multi_head_attention}
    \mathcal{MA}_\theta (x_i,X) = x_i + \sum_{h=1}^H W^h \mathcal{A}_{\theta^h}(x_i,X) \,.
\end{equation}
An MLP with parameters $\xi$ is defined as a function that maps each token individually and independently through:
\begin{equation}\label{eq:mlp_def}
    F_{\xi}(x_i,X) = F_{\xi}(x_i)\,.
\end{equation}
Then, a Transformer model with $L$ layers and parameters $(\theta_i,\xi_i)_{i=1}^L$ is the composition of attention in-context mappings and MLP layers:
\begin{equation}\label{eq:transf_comp_tokens}
    \mathcal{T}_{(\theta_i,\xi_i)_{i=1}^L} \coloneqq F_{\xi_1} \circ \mathcal{MA}_{\theta_1} \circ \cdots \circ F_{\xi_L} \circ \mathcal{MA}_{\theta_L} \,.
\end{equation}

\subsection{Mean-field limit, measure-theoretic Transformers}

Note that Formula \eqref{eq:multi_head_attention} makes sense for an arbitrary number $n$ of tokens. To extend this formulation for an arbitrary, possibly infinite, number of tokens, it is convenient to define a mean-field version of Transformers. For a general probability measure $\mu \in \mathcal{P}(\R^d)$, the corresponding contextual mapping is defined at each point $x \in \R^d$ by:
\begin{equation}\label{eq:mean_field_attention}
    \Gamma_{\theta}(x,\mu) \coloneqq x + \sum_{h=1}^H W^h \int \frac{ S(Q^h x, K^h y)}{\int N(Q^hx,K^hz)d\mu(z)} V^h y d\mu(y) \,,
\end{equation}
where $S$ is a similarity or relation function, and $N$ is a normalization function.

Note that in both cases, the defined mapping is a strict generalization of the standard token attention layer. Indeed, if $X = (x_1, \ldots, x_n)$ is a token sequence, then
\begin{equation}
    \mathcal{MA}_{\theta}(x,X) = \Gamma_{\theta}\left( x, \frac{1}{n}\sum_{j=1}^n \delta_{x_j} \right) \,.
\end{equation}
Then, a multi-head attention layer can be interpreted as the push-forward of a measure $\mu$ through the $\mu$-dependent contextual mapping $\Gamma_{\theta}(\mu) :x \mapsto \Gamma_{\theta}\left( x,\mu \right)$. The token mapping $X=(x_i, 1\leq i \leq n) \mapsto (\mathcal{MA}_{\theta}(x_i,X), 1 \leq i \leq n)$ is a special case of the measure mapping
\begin{equation}
    \mu \mapsto \Gamma_{\theta}(\mu)\#\mu\,.
\end{equation}
Following the notations from \cite{furuya2024transformersuniversalincontextlearners}, the composition of layers in a Transformer can be written as the composition of push-forward operations via
\begin{equation}
    (\Gamma_{\theta_2} \diamond \Gamma_{\theta_1})(x,\mu) \coloneqq \Gamma_{\theta_2}( \Gamma_{\theta_1}(x,\mu), \Gamma_{\theta_1}(\mu)\# \mu) \,. 
\end{equation}
The MLP layers $F_{\xi}$ are interpreted as mappings independent of the context, i.e. $F_{\xi}(x,\mu) = F_{\xi}(x)$, so that a measure-theoretic Transformer with $L$ layers from $\mathcal{P}(\R^d)$ to $\mathcal{P}(\R^d)$ is the composition
\begin{equation}
    F_{\xi_L} \diamond \Gamma_{\theta_L} \diamond \cdots \diamond F_{\xi_1} \diamond \Gamma_{\theta_1}  \,.
\end{equation}
This definition is equivalent to Formula \ref{eq:transf_comp_tokens}. For further details and applications of this approach, we refer to \cite{sander2022sinkformerstransformersdoublystochastic,castin2024smoothattention, furuya2024transformersuniversalincontextlearners}.

\subsection{Proof of Proposition \ref{prop:transf_universality_measure}}\label{AppendixFirstProof}

Following the methods of \cite{furuya2024transformersuniversalincontextlearners}, we prove the density of a particular function class.
Let us consider the elementary mappings 
\begin{equation}
    \gamma_{\lambda}(x,\mu) \coloneqq \langle x,a \rangle + b + \int v\operatorname{ReLU}(\langle x,a \rangle - \langle y,a \rangle + c)d\mu(y) \,,
\end{equation}
where $\lambda \coloneqq (a,b,c,v) \in \R^d \times \R \times \R \times \R$. These mappings are obtained by composing a linear MLP and a simple attention head, operating in dimension one. Define the affine MLP $F_{\xi}(x) = \langle x,a \rangle + b$, the attention parameters $\theta = (a_q,b_q,a_k,b_k,a_v,b_v) \in \R^6$, and a simple MLP-attention composition by
\begin{equation}
    \Gamma_\theta \diamond F_{\xi} (x,\mu) = \langle x,a \rangle + b + \int (a_v (\langle a,y \rangle + b) + b_v) \operatorname{ReLU}\left[a_q\langle x,a \rangle - a_k\langle y,a \rangle + b_q - b_k\right] d\mu(y) \,.
\end{equation} 
Then if $\lambda = (a,b,c,v)$, setting $\xi = (a,b)$ and $\theta = (1,c,1,0,0,v)$, we get
\begin{equation}
    \gamma_{\lambda} = \Gamma_{\theta} \diamond F_{\xi} \,,
\end{equation}
Now, define the algebra spanned by these elementary mappings:
\begin{equation}
    \mathcal{A} \coloneqq \left\{ \Omega \times\mathcal{P}(\Omega) \ni (x,\mu) \mapsto \sum_{n=1}^N \gamma_{\lambda_{1,n} }\odot \dots \odot \gamma_{\lambda_{T,n} } \,;\, T,N \in \mathbb{N}\right\}\,.
\end{equation}
The rest of the proof of Proposition \ref{prop:transf_universality_measure} is the same as Theorem 1 an \cite{furuya2024transformersuniversalincontextlearners}, since only the definition of the base algebra $\mathcal{A}$ is different, once we have the following result:
\begin{proposition}\label{ThDenseAlgebra}
    The algebra $\mathcal{A}$ is dense in the space of $( l^2 \times\text{weak}^*)$-continuous functions from $\Omega \times \mathcal{P}(\Omega) $ to $\R$. 
\end{proposition}

\begin{proof}
    First, note that if $\Omega$ is compact, then $\Omega \times \mathcal{P}(\Omega)$ is also compact, and that functions in $\mathcal{A}$ are continuous. Furthermore, setting $a=v=0$ and $b=1$ shows that non-zero constant functions are included in $\mathcal{A}$. Now, let us prove that $\mathcal{A}$ separates points in $\Omega \times \mathcal{P}(\Omega)$, in order to use the Stone-Weierstrass theorem. 
    
    Let $(x,\mu),(x',\mu') \in \Omega \times \mathcal{P}(\Omega)$, and suppose that for all $\lambda = (a,b,c,v)$, we have $\gamma_{\lambda}(\mu,x) = \gamma_{\lambda}(\mu',x')$. First, setting $v = 0$ implies that for all $a \in \R^d$, $\langle x,a\rangle = \langle x',a \rangle$, so that $x = x'$. To prove that $\mu = \mu'$, we consider the operator:
    \begin{equation}
        L_{x,a}(\mu)(c) \coloneqq \int \operatorname{ReLU}(\langle a, x-y \rangle + c) d\mu(y) \,.
    \end{equation} 
    Let $\mu,\mu' \in \mathcal{P}(\Omega)$, and suppose that for all $\lambda = (a,b,c,v)$, we have $\gamma_{\lambda}(x,\mu) = \gamma_{\lambda}(x,\mu')$. This equality implies that for all $a,c$, we have $L_{x,a}(\mu)(c) = L_{x,a}(\mu')(c)$.
    Let us begin with the one-dimensional case, supposing that $\Omega \subset \R$.
    In dimension 1, the operator $L_{x,1}(\mu)$ is linear in $\mu$, differentiable in $c$ almost everywhere and verifies: 
    \begin{equation}
        L_{x,1}(\mu)'(x) = \int_{-\infty}^{x+c} d\mu(y) \,.
    \end{equation}
    This proves the injectivity of $L_{x,a}$ in the one-dimensional case, by injectivity of the cumulative distribution function. Now, for the general case $\Omega \subset \R^d$, we can see that the higher dimension operator corresponds to its one-dimensional version applied to the one-dimensional measure $P_a \# \mu$, where $P_a : x \mapsto \langle x,a \rangle$. This proves that $P_a \# \mu = P_a \# \mu'$, for all $a \in \R^d$, so that $\mu = \mu'$ by injectivity of the Radon transform. All the conditions to apply the Stone-Weierstrass theorem to $\mathcal{A}$ are met, completing the proof.
\end{proof}

\section{Hyperparameters for LRA}\label{app:lra_hyper}

The hyperparameters used for the LRA experiments are collected in table \ref{tab:lra_hparams_tasks}.

\begin{table}[t]
\centering
\begin{tabular}{l|ccccc}
\toprule
\textbf{Hyperparameter} 
& \textbf{ListOps}
& \textbf{Text}
& \textbf{Retrieval}
& \textbf{Image}
& \textbf{Pathfinder} \\
\midrule
Layers          & 6 & 8 & 8 & 6 & 6 \\
Hidden size     & 256 & 256 & 128 & 256 & 128 \\
Attention heads & 4 & 4 & 4 & 4 & 4 \\
FFN dimension   & 512 & 512 & 256 & 512 & 256 \\
Pos.~encoding   & RoPE & RoPE & RoPE & Learned & Learned \\
\bottomrule
\end{tabular}
\caption{Hyperparameters used for the LRA experiments. 
All tasks use mean pooling and both models share identical settings.}
\label{tab:lra_hparams_tasks}
\end{table}

\vspace{0.1cm}

\section{The hat kernel}\label{AppendixWendland}

The hat kernel belongs to the family of Wendland kernels:

\begin{proposition}[Wendland Kernels]
Let $l \in \mathbb{N}$, define
 $\phi_{l} : [0,\infty) \to \mathbb{R}$ be given by
\begin{equation}
    \phi_{l}(r) = ((1-r)_+)^l\,.
\end{equation}
Then for all $l \geq \left\lfloor \frac{d}{2} \right\rfloor + 1$, the radial function
\begin{equation}
\Phi(x) = \phi_{l}(\|x\|), \quad x \in \mathbb{R}^d,
\end{equation}
is positive definite on $\mathbb{R}^d$, compactly supported in the unit ball, and continuous.
\end{proposition}
The hat kernel or ReLU-bump is the case $l = 1$, and in this case, the fact that the kernel is positive definite in one dimension can be obtained since $\phi_1 = \mathbf 1_{[-1/2,1/2]} \star \mathbf 1_{[-1/2,1/2]}$ where $\star$ denotes the convolution. Therefore, its Fourier transform is $\widehat \phi_1(\omega) = \operatorname{sinc}(\omega)^2$. 
From Section \ref{SecQuasiLinearComputation}, this 1D kernel can be computed exactly in $n \log n$ operations.
Using a 1D positive definite kernel, one can construct positive definite kernels in any dimension $d$ by projecting with orthogonal projections $\pi_\theta$ and taking expectations on the direction $\theta \in \mathbb S^d$ under the uniform measure on the sphere: for $x,y \in \mathbb R^d$
\begin{equation}\label{EqSlicingKernels}
    K(x,y) \coloneqq \mathbb E_{\mathbb S^{d-1}}[k(\pi_\theta(x),\pi_\theta(y))]\,.
\end{equation}
Now, for translation-invariant 1D kernels $k(x,y) = \phi(|x - y|)$, the resulting kernel $K$ is translation and rotation-invariant, and thus can be expressed as $K(x,y) = \Phi(\|x - y\|)$ where 
\begin{equation}
    \Phi(r) = c \int_0^1 \phi(tr) (1-t^2)^{\frac{d-3}{2}}dt\,,
\end{equation}
and $c = \frac{2\Gamma(d/2)}{\sqrt{\pi} \Gamma(\frac{d-1}{2})}$, where $\Gamma$ is the Gamma function. Although an explicit formula can be computed explicitly with the Beta function in the case $\phi_1$, the approximation scheme consists in using Monte-Carlo methods on Formula \eqref{EqSlicingKernels} to benefit from the $O(n \log n)$ computation in 1D, thereby introducing a mean square error which usually decreases as $1/\sqrt{p}$ where $p$ is the number of slices.
To introduce more variety in the shape of such kernels, one can use a mixture of scaled versions of $\phi_1$, that is $\sum_{i = 1}^k \beta_i \phi_1(r/\alpha_i)$ where $\alpha_i,\beta_i$ are positive real parameters. For more detailed discussions on slicing kernels, we refer the reader to \cite{NEURIPS2019_f0935e4c,hertrich2024generativeslicedmmdflows}.

\vspace{0.1cm}

\section{Other experiments with tiny ViT}\label{AppendixTinyViT}

In Figure~\ref{FigureOtherExperiments}, we evaluate a sliced variant of softmax attention and compare it to standard softmax attention. The sliced softmax yields lower accuracy on CIFAR-10. Given the small size of this dataset, one might expect that introducing additional inductive bias would improve performance. However, this is not observed for softmax attention. This suggests that the modest performance gains observed with sliced ReLU attention cannot be attributed solely to increased inductive bias.
Also in Figure~\ref{FigureOtherExperiments}, we report, for completeness, the best results we obtained with linear attention. In our experiments, this approach required careful tuning and exhibited sensitivity to training hyperparameters.

\begin{figure}[h]
    \centering
    \includegraphics[width=0.8\linewidth]{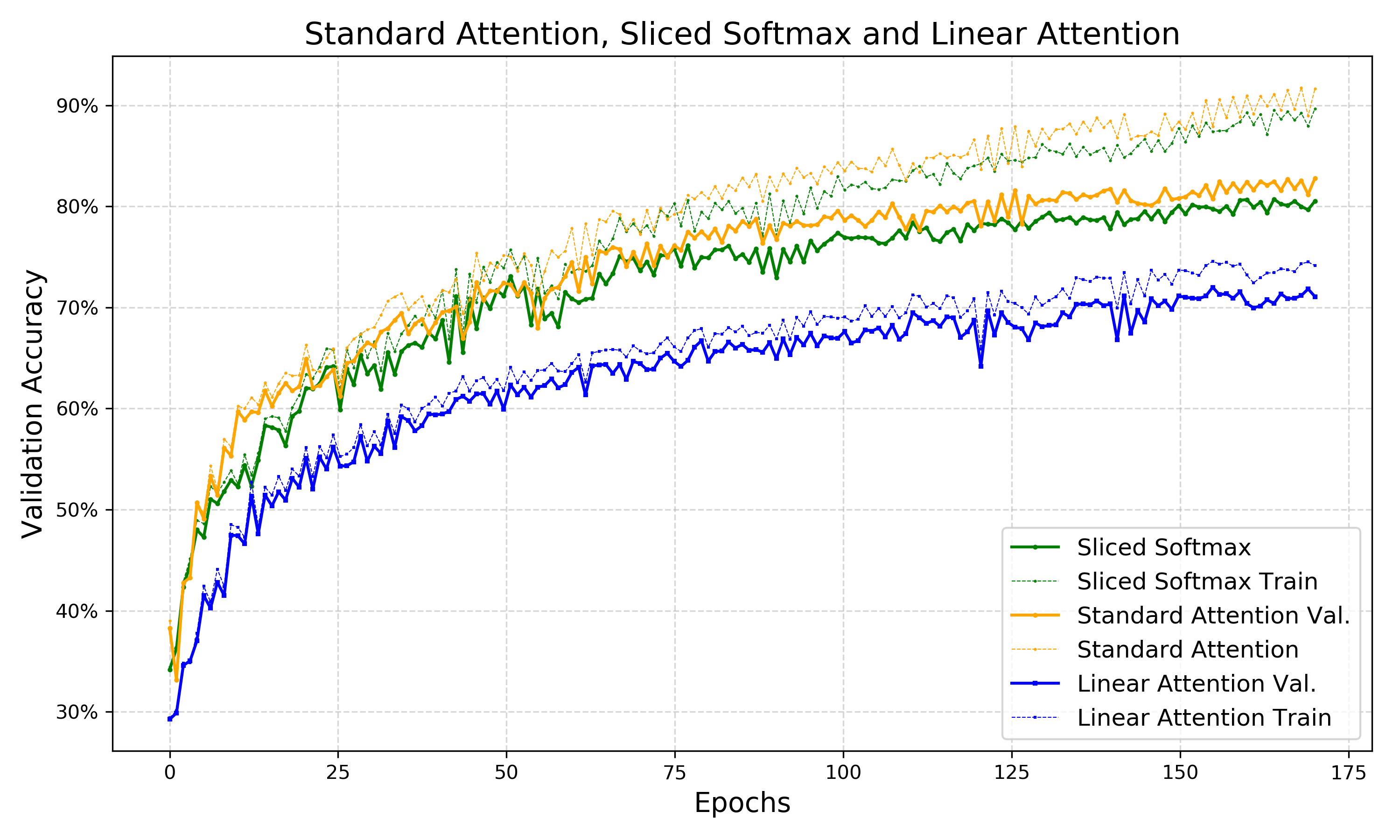}
    \caption{Slicing the softmax attention results in lower performance on CIFAR-10. For each setup, we selected the best hyperparameters from a small grid search, and report the median result over four independent runs.}
    \label{FigureOtherExperiments}
\end{figure}

\section{Structure of Sliced ReLU-Bump Attention}
\label{app:visu_band}

In this section, we provide a qualitative visualization of the structure induced by sliced ReLU-bump attention after MLM pretraining. Sliced ReLU-bump attention operates on a restricted set of keys determined by a data-dependent bump window, parametrized by an effective bandwidth. This design gives rise to interpretable geometric quantities, such as the effective bandwidth per head and the number of query/key token interactions. All visualizations in this section are obtained from the sliced ReLU-bump model after MLM pretraining.

We first examine aggregate statistics characterizing the global behavior of sliced ReLU-bump attention across layers and heads.
In particular, we focus on the effective bandwidth of the bump window and on the mean number of keys selected per query token, averaged over a random subset of the MLM training corpus.
These quantities provide a coarse but stable description of how attention locality and sparsity evolve throughout the network.

\begin{figure}[htbp]
    \centering
    \begin{subfigure}[t]{0.49\linewidth}
        \centering
        \includegraphics[width=\linewidth]{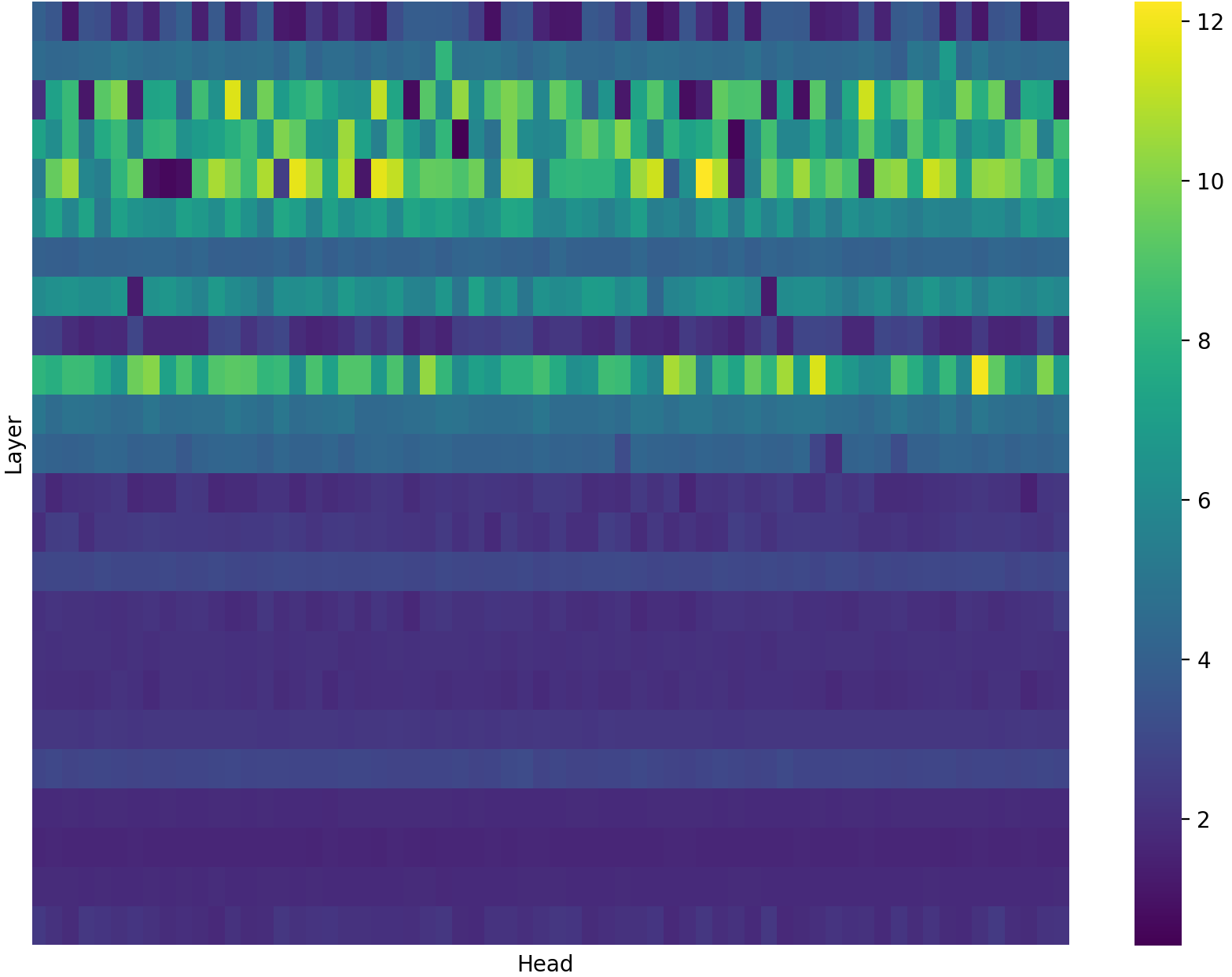}
        \caption{Effective bandwidth.}
        \label{fig:relubump_bw_heatmap}
    \end{subfigure}\hfill
    \begin{subfigure}[t]{0.49\linewidth}
        \centering
        \includegraphics[width=\linewidth]{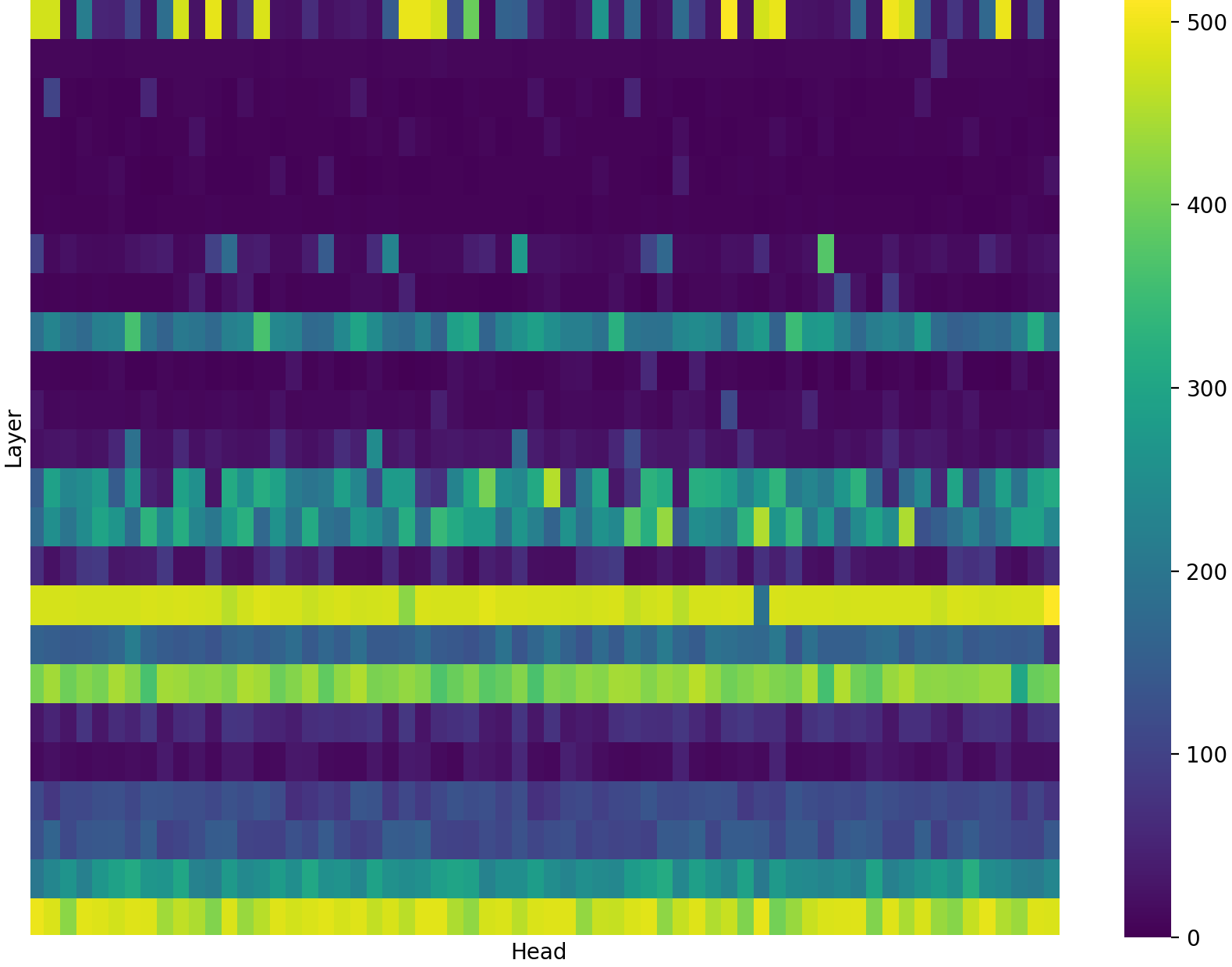}
        \caption{Mean number of keys in the bump window.}
        \label{fig:relubump_win_keys_heatmap}
    \end{subfigure}

    \caption{
    Structure of sliced ReLU-bump attention across layers and heads.
    Rows correspond to layers, ordered from early (top) to deep (bottom), and columns correspond to attention heads.
    Statistics are computed over a random text batch from the MLM pretraining dataset.
    }
    \label{fig:relubump_aggregate_heatmaps}
\end{figure}

Figure~\ref{fig:relubump_aggregate_heatmaps} shows a non-uniform organization of sliced ReLU-bump attention across layers and heads.
The first layers (top rows) exhibit a more heterogeneous distribution of effective bandwidths across heads, while deeper layers tend to display more homogeneous bandwidth profiles.
Variations in effective bandwidth do not directly translate into variations in the number of active keys, as shown by the distinct patterns observed in the bump window size heatmap.

We also visualize attention maps produced by sliced ReLU-bump attention on natural language inputs. For each example, we compute the full quadratic attention matrix for a fixed head and layer, rather than relying on the sliced computation, in order to directly examine the induced attention patterns.

Figure~\ref{fig:relubump_attention_examples} illustrates several distinct attention regimes, including local banded patterns, structured sparse interactions, clause-level block structures, and multi-block configurations.
These patterns are observed after MLM pretraining, without any explicit architectural constraint enforcing such behavior.

Overall, these observations indicate that sliced ReLU-bump attention is compatible with a variety of interaction structures while retaining its quasi-linear formulation.

\clearpage

\end{document}